\documentclass[a4paper,11pt]{scrartcl}

\usepackage[T1]{fontenc}
\usepackage[utf8]{inputenc}
\usepackage{lmodern}
\usepackage{amsfonts,amsmath,amssymb,amsthm,epsfig,graphicx,booktabs,latexsym,mathtools} %
\usepackage[margin=1.00in]{geometry}
\usepackage[english]{babel}

\usepackage{wrapfig}
\usepackage{mdframed}
\usepackage[authoryear,round]{natbib}
\usepackage{tikz}
\usetikzlibrary{positioning,shadows,arrows,shapes,shapes.arrows,arrows.meta,trees,shapes.misc,shapes.geometric,decorations.pathreplacing,calc}
\tikzset{
    -Latex,auto,node distance =1 cm and 1 cm,semithick,
    state/.style ={circle, draw=white, inner sep=0.01cm, minimum width = 0.35 cm},
    snode/.style = {rectangle, draw, inner sep=0.1cm,node contents={}},
    unmeasured/.style = {circle, draw, inner sep=0.05cm,node contents={}},
    point/.style = {circle, draw, inner sep=0.04cm,fill,node contents={}},
    bidirected/.style={Latex-Latex,dashed},
    el/.style = {inner sep=2pt, align=left, sloped}
}
\tikzstyle{mybox} = [draw=gray, fill=gray!20, very thick,
rectangle, rounded corners, inner xsep=3pt, inner ysep=7pt]
\usepackage{bm}

\usepackage[on]{settings/editing}
\usepackage{cancel}
\usepackage{multirow}
\usepackage{color}
\usepackage{authblk,epstopdf}
\usepackage{array} 
\usepackage{verbatim} 
\usepackage[normalem]{ulem}
\usepackage{thmtools}
\usepackage[T1]{fontenc}
\usepackage[utf8]{inputenc}
\usepackage{lmodern}
\usepackage{setspace}
\usepackage{adjustbox}
\usepackage{graphicx}
\usepackage{subfig}
\usepackage[useregional]{datetime2}
\usepackage{clipboard}
\usepackage{enumitem,kantlipsum} 
\usepackage[some]{background}
\usepackage{optidef}
\usepackage[most]{tcolorbox}
\tcbuselibrary{theorems}
\usepackage{longtable}
\usepackage{ltablex}
\usepackage{nccmath}
\usepackage{pifont}
\usepackage{sansmath}
\usepackage{booktabs,tabularx}    
\usepackage{makecell}  

\newcommand{\cmark}{\ding{51}}  
\newcommand{\xmark}{\ding{55}}  

\definecolor{betterred}{RGB}{228,26,28}
\definecolor{betterblue}{RGB}{55,126,184}

\newcommand\labelAndRemember[2]
  {\expandafter\gdef\csname labeled:#1\endcsname{#2}%
   \label{#1}#2}
\newcommand\recallLabel[1]
   {\csname labeled:#1\endcsname\tag{\ref{#1}}}

\newcommand{\Ind}{\boldsymbol{1}}

\let\oldnl\nl
\newcommand{\nonl}{\renewcommand{\nl}{\let\nl\oldnl}}
\definecolor{opaquegray}{RGB}{192, 192, 192}
\definecolor{afblue}{RGB}{0,0,102}
\definecolor{mypink}{RGB}{255,51,153}

\newtheorem{definition}{Definition}
\newtheorem{theorem}{Theorem}
\newtheorem{assumption}{Assumption}

\newtheorem{lemma}{Lemma}

\newtheorem{proposition}{Proposition}
\newtheorem{corollary}{Corollary}

\tcolorboxenvironment{definition}{
  colback=gray!5!white,
  colframe=black,
  fonttitle=\bfseries,
  boxrule=0.5pt, 
  breakable,
}

\tcolorboxenvironment{theorem}{
  colback=gray!5!white,
  colframe=black,
  fonttitle=\bfseries,
  boxrule=0.5pt, 
  breakable,
}

\tcolorboxenvironment{corollary}{
  colback=gray!5!white,
  colframe=black,
  fonttitle=\bfseries,
  boxrule=0.5pt, 
  breakable,
}

\tcolorboxenvironment{lemma}{
  colback=gray!5!white,
  colframe=black,
  fonttitle=\bfseries,
  boxrule=0.5pt, 
  breakable,
}

\tcolorboxenvironment{proposition}{
  colback=gray!5!white,
  colframe=black,
  fonttitle=\bfseries,
  boxrule=0.5pt, 
  breakable,
}

\tcolorboxenvironment{proof}{
  colback=white,
  colframe=black,
  fonttitle=\bfseries,
  boxrule=0.5pt, 
  breakable,
}

\tcolorboxenvironment{example}{
  colback=white,
  colframe=black,
  fonttitle=\bfseries,
  boxrule=0.5pt, 
  breakable,
}

\tcolorboxenvironment{remark}{
  colback=gray!5!white,
  colframe=black,
  fonttitle=\bfseries,
  boxrule=0.5pt, 
  breakable,
}

\tcolorboxenvironment{assumption}{
  colback=white,
  colframe=black,
  fonttitle=\bfseries,
  boxrule=0.5pt, 
  breakable,
}

\usepackage[dvipsnames]{xcolor}
\usepackage[colorlinks=true, citecolor=afblue, linkcolor=RoyalBlue]{hyperref}%

\title{\Copy{title}{Information-Theoretic Causal Bounds under Unmeasured Confounding}}

\author{
Yonghan Jung$^{1}$ and Bogyeong Kang$^{2}$ \\
{\small $^{1}$ University of Illinois Urbana-Champaign,
\texttt{yonghan@illinois.edu} \\
$^{2}$ Independent Researcher,
\texttt{bgyeong.kang@gmail.com}}
}

\date{\today}
\begin{document}
\setbox0\hbox{\Copy{shorttitle}{Information-Theoretic Causal Bounds}}
\maketitle
\allowdisplaybreaks


\begin{abstract}
     We develop a data-driven information-theoretic framework for the sharp partial identification of causal effects under unmeasured confounding. Existing approaches often rely on restrictive assumptions, such as bounded or discrete outcomes, require external inputs (e.g., instrumental variables, proxies, or user-specified sensitivity parameters), necessitate full structural causal model specifications, or focus solely on population-level averages while neglecting covariate-conditional treatment effects. We overcome all four limitations simultaneously by establishing novel information-theoretic, data-driven divergence bounds. Our key theoretical contribution establishes that the $f$-divergence between the observational distribution $P(Y \mid A=a, X=x)$ and the interventional distribution $P(Y \mid \mathrm{do}(A=a), X=x)$ is upper bounded by a function of the propensity score alone. This result enables sharp partial identification of conditional causal effects directly from observational data, without requiring external sensitivity parameters, auxiliary variables, full structural specifications, or outcome boundedness assumptions. For practical implementation, we develop a semiparametric estimator satisfying Neyman-orthogonality \citep{chernozhukov2018double}, which ensures $\sqrt{n}$-consistent inference even when nuisance functions are estimated via flexible machine learning methods. Simulation studies and real-world data applications, implemented in \href{https://github.com/yonghanjung/Information-Theretic-Bounds}{GitHub repository}, demonstrate that our framework provides tight and valid causal bounds across a wide range of data-generating processes.

\end{abstract}

\tableofcontents
\newpage

\begingroup
\allowdisplaybreaks
\sloppy
\section{Introduction}
\label{sec:intro}

Causal effect identification aims to characterize interventional quantities, such as $\Pr(Y=y \mid \mathrm{do}(A=a),X=x)$, as functionals of the observational distribution $P(X,A,Y)$. In the presence of unmeasured confounders $U$, as depicted in Fig.~\ref{fig:dgp}, point identification is generally impossible without auxiliary variables or structural restrictions. In such settings, \textit{partial identification} seeks to recover bounds that provably contain the true causal quantity. However, as described in literature review in Sec.~\ref{subsec:related}, most existing methods suffer from one or more of the following fundamental limitations: 
\begin{enumerate}[label=\textbf{(Lim-\arabic*)}, leftmargin=*, topsep=0pt]
    \item \textbf{Bounded outcomes:} Restricting outcomes to bounded or discrete supports (e.g., $Y \in [0,1]$).
    \item \textbf{Externality of parameters:} Requiring auxiliary inputs—such as instrumental variables, proxies, or sensitivity parameters—to quantify confounding strength.
    \item \textbf{Full SCM specification:} Necessitating the specification of the entire structural causal model (SCM) \citep{pearl:2k}, which is computationally intensive and prone to error propagation.
    \item \textbf{Neglect of heterogeneity:} Focusing on population-level averages while neglecting covariate-conditional treatment effects.
\end{enumerate}

To universally address these limitations, we develop an information-theoretic framework that provides (i) data-driven upper bounds on statistical divergences between observational and interventional distributions, and (ii) sharp partial identification of conditional causal effects $\mathbb{E}[Y \mid \mathrm{do}(A=a),X=x]$. Our framework accommodates \textit{unbounded} continuous outcomes without requiring full structural modeling or external inputs. The core mechanism involves deriving \textit{data-driven} upper bounds on statistical divergences (e.g., f-divergence \citep{csiszar1967fdivergenceDPI}) between the interventional law $Q_{a,x}$ and the observational law $P_{a,x}$, and then translating these into sharp causal intervals. Specifically, we make three main contributions:
\begin{enumerate}[leftmargin=*, itemsep=0pt, topsep=0pt, label=(\roman*)]
    \item We show that $f$-divergences \citep{csiszar1967fdivergenceDPI} between $P_{a,x}$ and $Q_{a,x}$ are upper bounded by a function of the propensity score $e_a(x)$.
    \item We leverage these bounds to obtain sharp intervals for arbitrary expectations of the form $\theta(a,x)\triangleq \mathbb{E}_{Q_{a,x}}[\varphi(Y)]$ for user-specified functions $\varphi$ without imposing outcome boundedness or support restrictions.
    \item We develop a semiparametric estimator that satisfies Neyman-orthogonality \citep{chernozhukov2018double} ensuring robust inference even when nuisance components are estimated via high-dimensional machine learning models.
\end{enumerate}
Together, these results provide a principled path to \textit{data-driven} partial identification of conditional causal effects under unmeasured confounding.

\subsection{Related Work} \label{subsec:related}
We organize existing work on partial identification based on which of the limitations (Lim-1--4) they retain or address. 

\paragraph{Bounded/discrete outcomes (Lim-1).} Early work imposed restrictions requiring outcomes to be bounded or discrete. For example, \citet{manski1990nonparametric} derived nonparametric bounds using the extreme values outcomes can attain. Linear-programming (LP)–based approaches (e.g., \citet{balke1994counterfactual}, \citet{tian2000probabilities}) yield sharp bounds with discrete variables. \citet{sachs2023general} and \citet{shridharan2023scalable} have extended these LP-based bounds to general graphical settings but remain restricted to discrete outcomes. \citet{zhang2021bounding} extended these LP ideas to continuous outcomes, but still rely on bounded-support assumptions (e.g., $Y \in [0,1]$). These methods avoid auxiliary inputs (addressing Lim-2) but fail to accommodate unbounded outcomes (Lim-1) or provide conditional effect bounds (Lim-4).

\paragraph{Auxiliary inputs (Lim-2).} Another line of work leverages auxiliary inputs. While auxiliary-variable methods can yield sharp bounds, most methods still assume bounded outcomes (Lim-1); and valid auxiliary inputs are often not available in practice or not identifiable from data. 
\begin{itemize}[leftmargin=*, topsep=0pt, itemsep=0pt]
    \item \textit{Instrumental variables.} \citet{balke1997bounds} provide tight nonparametric bounds on average treatment effects by leveraging instrumental variables, assuming bounded binary outcomes. \citet{kitagawa2021IV} extends this framework to continuous outcomes while maintaining bounded support assumptions (see \citet{swanson2018partial} for a comprehensive survey). Recently, \citet{levis2025covariate} develop covariate-assisted IV bounds to target conditional treatment effects (addressing Lim-4), but also under bounded outcome assumptions (Lim-1).
    \item \textit{Additional assumptions or variables.} \citet{ghassami2023partial} leverage proxy variables of hidden confounders to provide bounds on average effects, again requiring bounded outcomes (Lim-1). \citet{lee2009training} and \citet{semenova2025generalized} avoid bounded outcomes for sharp bounds on the average effect, but rely on structural assumptions about the selection mechanism. Recent work on weak-confounding-aware partial identification also relies on auxiliary side information in the form of bounded latent-confounder entropy. Jiang et al.~\citep{pmlr-v202-jiang23h} study entropy/mutual-information--constrained bounds for non-identifiable causal effects in a general latent-confounding setting, and Jiang and Kocaoglu~\citep{pmlr-v235-jiang24b} extend this entropy-based perspective to IV settings via conditional common entropy, including a necessary falsification condition when an entropy bound is available.
    \item \textit{Sensitivity analysis.} Sensitivity analysis introduces user-specified parameters to quantify confounding strength (e.g., \citet{rosenbaum1987sensitivity}, \citet{tan2006distributional}, \citet{yadlowsky2022bounds}, \citet{jin2022sensitivity}, \citet{dorn2023sharp}, \citet{oprescu2023b}). Unlike IV and proxy methods, modern sensitivity approaches can accommodate unbounded outcomes (addressing Lim-1). Among these, \citet{jin2022sensitivity} are most closely related to our approach, as they use an $f$-divergence-based sensitivity model to constrain divergences between observational and interventional distributions. \citet{oprescu2023b} extend sensitivity analysis to bound conditional effects (addressing Lim-4). However, all sensitivity methods require external sensitivity parameters (Lim-2) that are not identifiable from observational data alone.
\end{itemize}

\begin{figure}[t]
  \centering
  \begin{minipage}[c]{0.25\textwidth}
    \centering
    \subfloat[]{%
      \begin{tikzpicture}[x=1.0cm,y=1.4cm,>={Latex[width=1.4mm,length=1.7mm]},
        font=\sffamily\sansmath\scriptsize,
        RR/.style={draw,circle,inner sep=0mm, minimum size=5.8mm,font=\sffamily\sansmath\footnotesize}]%
        \pgfmathsetmacro{\u}{0.95}%
        \node[RR, betterblue] (a) at (0*\u, 0*\u) {$A$};%
        \node[RR, lightgray] (u) at (1.5*\u, 2*\u) {$U$};%
        \node[RR, betterred] (y) at (3*\u, 0*\u) {$Y$};%
        \node[RR] (x) at (1.5*\u, 1*\u) {$X$};%
        \draw[->] (x) -- (a);%
        \draw[->] (x) -- (y);%
        \draw[->] (a) -- (y);%
        \draw[->, lightgray] (u) -- (x);%
        \draw[->, lightgray] (u) -- (a);%
        \draw[->, lightgray] (u) -- (y);%
      \end{tikzpicture}%
      \label{fig:dgp}%
    }%
  \end{minipage}
  \hfill 
  \begin{minipage}[c]{0.70\textwidth}
    \centering
    \subfloat[]{%
      \begin{tikzpicture}[
    node distance=0mm,
    every node/.style={font=\sffamily\scriptsize, align=center},
    header/.style={font=\sffamily\bfseries\scriptsize, text=betterblue},
    cell/.style={inner sep=2pt, minimum height=0.55cm},
    check/.style={text=betterblue!80!black},
    cross/.style={text=betterred!90!black}
]

\def\xC{0}       
\def\xLine{3.2}  
\def\xU{4.2}     
\def\xA{6.4}     
\def\xS{8.6}     
\def\xE{10.8}    
\def\xEnd{11.9}  

\def\yH{0}
\def\yI{-0.75}
\def\yII{-1.5}
\def\yIII{-2.25}
\def\yIV{-3.0}
\def\yV{-3.75}

\node[cell, anchor=west, font=\sffamily\bfseries\scriptsize, text=black] at (\xC, \yH) {Method};
\node[cell, header] at (\xU, \yH) {Unbounded\\Outcome};
\node[cell, header] at (\xA, \yH) {No Aux\\Input};
\node[cell, header] at (\xS, \yH) {No Full\\SCM};
\node[cell, header] at (\xE, \yH) {Cond.\\Effect};

\draw[betterblue, thick, -] (\xC, \yH-0.35) -- (\xEnd, \yH-0.35);

\draw[gray!50, -] (\xLine, 0.2) -- (\xLine, -4.1);

\node[cell, anchor=west, text width=3.0cm, align=left] at (\xC, \yI) {LP / Discrete};
\node[cell, cross] at (\xU, \yI) {\xmark};
\node[cell, check] at (\xA, \yI) {\cmark};
\node[cell, check] at (\xS, \yI) {\cmark};
\node[cell, cross] at (\xE, \yI) {\xmark};
\draw[gray!20, -] (\xC, \yI-0.375) -- (\xEnd, \yI-0.375);

\node[cell, anchor=west, text width=3.0cm, align=left] at (\xC, \yII) {Additional Vars. \\ \tiny (IV, Proxy)};
\node[cell, cross] at (\xU, \yII) {\xmark};
\node[cell, cross] at (\xA, \yII) {\xmark};
\node[cell, check] at (\xS, \yII) {\cmark};
\node[cell, check] at (\xE, \yII) {\cmark};
\draw[gray!20, -] (\xC, \yII-0.375) -- (\xEnd, \yII-0.375);

\node[cell, anchor=west, text width=3.0cm, align=left] at (\xC, \yIII) {Sensitivity};
\node[cell, check] at (\xU, \yIII) {\cmark};
\node[cell, cross] at (\xA, \yIII) {\xmark};
\node[cell, check] at (\xS, \yIII) {\cmark};
\node[cell, check] at (\xE, \yIII) {\cmark};
\draw[gray!20, -] (\xC, \yIII-0.375) -- (\xEnd, \yIII-0.375);

\node[cell, anchor=west, text width=3.0cm, align=left] at (\xC, \yIV) {Full SCM};
\node[cell, check] at (\xU, \yIV) {\cmark};
\node[cell, check] at (\xA, \yIV) {\cmark};
\node[cell, cross] at (\xS, \yIV) {\xmark};
\node[cell, check] at (\xE, \yIV) {\cmark};

\fill[betterblue!8, rounded corners=4pt] (\xC, \yV+0.375) rectangle (\xEnd, \yV-0.375);
\node[cell, anchor=west, text width=3.0cm, align=left, font=\sffamily\bfseries\small, text=black] at (\xC, \yV) {Ours};
\node[cell, check] at (\xU, \yV) {\cmark};
\node[cell, check] at (\xA, \yV) {\cmark};
\node[cell, check] at (\xS, \yV) {\cmark};
\node[cell, check] at (\xE, \yV) {\cmark};

\end{tikzpicture}%
      \label{fig:positioning}%
    }%
  \end{minipage}

  \caption{\textbf{(a)} Causal diagram with unmeasured confounding.
  \textbf{(b)} Systematic comparison of our method against existing literature (detailed in Sec.~\ref{subsec:related}).}
  \label{fig:positioning_and_dgp}
\end{figure}

\paragraph{Full SCM-modeling approaches (Lim-3).} 
Another approach leverages machine-learning methods to learn entire SCMs consistent with observational data (e.g., \citet{hu2021generative}, \citet{balazadeh2022partial}, \citet{padh2023stochastic}, \citet{xia2022neural}, \citet{tan2024consistency}). These approaches find the SCMs that maximize/minimize the target causal effect subject to observations, using flexible neural architectures to model structural functions. In principle, such methods can accommodate unbounded outcomes and target conditional effects (addressing Lim-(1,4)). However, they require estimating the entire SCM (Lim-3), which is computationally intensive and sensitive to misspecification in high-dimensional structural components.

\paragraph{Our novelty.} Existing methods each resolve some limitations; however, no existing approach overcomes all four limitations (Lim-1--4) universally. In contrast, our work simultaneously addresses all four limitations by developing bounds that (Lim-1) accommodate unbounded continuous outcomes without support restrictions; (Lim-2) require no auxiliary variables or sensitivity parameters; (Lim-3) avoid full SCM modeling; and (Lim-4) provide bounds for conditional effects $\mathbb{E}[Y \mid \mathrm{do}(A=a),X=x]$ beyond the population-level average. We compare our work with representative existing methods in Fig.~\ref{fig:positioning}.

\section{Problem Setup \& Preliminaries}
Consider a treatment $A \in \{0, 1\}$, a covariate vector $X \in \mathcal{X} \subseteq \mathbb{R}^{d_x}$, and an outcome $Y \in \mathcal{Y} \subseteq \mathbb{R}^{d_y}$. We consider the structural causal model (SCM) framework \citep{pearl:2k} as the data-generating process (DGP) for $(X, A, Y)$:
\begin{align}\label{def:scm}
    U \leftarrow f_U(\epsilon_U), \;\; X \leftarrow f_X(U, \epsilon_X), \;\;A \leftarrow f_A(X, U, \epsilon_A), \;\;Y \leftarrow f_Y(X, A, U, \epsilon_Y),
\end{align}
where $U$ represents unmeasured confounding, $f_{(\cdot)}$ are unknown structural functions, and $(\epsilon_U, \epsilon_X, \epsilon_A, \epsilon_Y)$ are mutually independent exogenous noise variables. The causal diagram induced by this SCM is depicted in Fig.~\ref{fig:dgp}. 

The operation $\mathrm{do}(A=a)$ denotes an intervention that replaces $f_A$ with a constant $a \in \{0,1\}$, while keeping the other structural equations invariant. For each $(a, x) \in \{0, 1\} \times \mathcal{X}$, we define the following conditional probability laws on $\mathcal{Y}$:
\begin{itemize}[leftmargin=*,labelsep=0.5em]
    \item \textbf{Observational Law}: $P_{a, x} \equiv P(Y \mid A=a, X=x)$, which is identifiable from data.
    \item \textbf{Interventional Law}: $Q_{a, x} \equiv P(Y \mid \mathrm{do}(A=a), X=x)$, our target of interest.
\end{itemize}
Under unmeasured confounding (i.e., when $f_A$ and $f_Y$ share $U$ as a common hidden parent), the interventional law $Q_{a, x}$ is unidentifiable from the observational law $P_{a, x}$. Consequently, any causal functional $\theta = \mathbb{E}_{Q_{a,x}}[\varphi(Y)]$ for some user-specified $\varphi$ (e.g., the identity for ATE/CATE) is also unidentifiable.

\paragraph{f-Divergence.}
To characterize the ``distance'' between the identifiable $P_{a,x}$ and the unidentifiable $Q_{a,x}$, we use $f$-divergences \citep{ali1966general,csiszar1967fdivergenceDPI}.
\begin{definition}[\textbf{f-Divergence}]\label{def:f-div}
    \normalfont
    Let $P$ and $Q$ be probability measures on $(\mathcal{Y}, \mathcal{F})$ such that $P \ll Q$. For a convex function $f:[0, \infty) \rightarrow \mathbb{R}$ with $f(1)=0$, the $f$-divergence of $P$ from $Q$ is 
    \begin{align}
        D_f(P \| Q) \triangleq \int_{\mathcal{Y}} f\left(\frac{dP}{dQ}\right) dQ.
    \end{align}
\end{definition}
Common specializations of $f$-divergence are as follows. Let $p$ and $q$ be the Radon–Nikodym derivatives of $P$ and $Q$ with respect to a common dominating measure $\mu$ (e.g., Lebesgue or counting measure).
\begin{itemize}[leftmargin=*,itemsep=0pt]
    \item \textbf{Kullback-Leibler (KL).} $f(t) \triangleq t \log t$ with $f(0) = 0$. Then, 
    \begin{align}
        D_{\mathrm{KL}}(P \| Q) = \int_{\mathcal{Y}} \log \left(\frac{dP}{dQ}\right) dP = \int_{\mathcal{Y}} p(y) \log \left( \frac{p(y)}{q(y)} \right) d\mu(y).
    \end{align}
    \item \textbf{Hellinger distance.} $f(t) \triangleq \frac{1}{2}(\sqrt{t}-1)^2$ with $f(0) = 1/2$. 
    \begin{align}
        D_{\mathrm{H}}(P \| Q) = \frac{1}{2}\int_{\mathcal{Y}} \left(\sqrt{\frac{dP}{dQ}} - 1 \right)^2 dQ = 1-\int_{\mathcal{Y}}\sqrt{p(y)q(y)} d\mu(y).
    \end{align}
    \item \textbf{$\chi^2$-divergence.} $f(t) \triangleq \frac{1}{2}(t-1)^2$ with $f(0) = 1/2$. 
    \begin{align}
        D_{\chi^2}(P \| Q) = \frac{1}{2}\int_{\mathcal{Y}} \left(\frac{dP}{dQ} - 1 \right)^2 dQ = \frac{1}{2}\int_{\mathcal{Y}} \frac{(p(y) - q(y))^2}{q(y)} d\mu(y).
    \end{align}
    \item \textbf{Total variation (TV).} $f(t) \triangleq \frac{1}{2}|t-1|$ with $f(0) = 1/2$. 
    \begin{align}
        D_{\mathrm{TV}}(P \| Q) = \frac{1}{2}\int_{\mathcal{Y}} |p(y) - q(y)| d\mu(y) = \sup_{B \in \mathcal{F}} |P(B) - Q(B)|.
    \end{align}
    \item \textbf{Jensen-Shannon.} $f(t) \triangleq \frac12 (t\log t - (t+1)\log(\frac{t+1}{2}))$ with $f(0) = \frac{1}{2}\log 2$. Let $M \triangleq \frac{P+Q}{2}$. 
    \begin{align}
        D_{\mathrm{JS}}(P \| Q) = \frac{1}{2}D_{\mathrm{KL}}(P \| M) + \frac{1}{2}D_{\mathrm{KL}}(Q \| M).
    \end{align}
\end{itemize}

\paragraph{Integral Probability Metrics (IPMs) \& Maximum Mean Discrepancy (MMD).}
Beyond the $f$-divergence, the integral probability metric (IPM; \citealt{muller1997integral}) and maximum mean discrepancy (MMD; \citealt{gretton2012kernel}) are commonly used. Let $\Phi \triangleq \{\varphi: \mathcal{Y} \mapsto [0,1]\}$  be a class of measurable functions. Then, each measure is defined as follows. 
\begin{definition}[\textbf{Integral Probability Metric (IPM)} \citep{muller1997integral}]
\label{def:ipm}
\normalfont
Let $P$ and $Q$ be probability measures on a measurable space $(\mathcal{Y}, \mathcal{F})$. Let $\Phi$ be a class of measurable real-valued functions on $\mathcal{Y}$. The integral probability metric (IPM) is 
\begin{align}
    D_{\mathrm{IPM},\Phi}(P \| Q) \triangleq  \sup_{\varphi \in \Phi} \left| \mathbb{E}_{P}[\varphi(Y)] - \mathbb{E}_{Q}[\varphi(Y)] \right|.
\end{align}
\end{definition}
\begin{definition}[\textbf{Maximum Mean Discrepancy (MMD)} \citep{gretton2012kernel}]
\label{def:mmd}
\normalfont
Let $\mathcal{H}_k$ be a reproducing kernel Hilbert space (RKHS) associated with a positive-definite kernel $k:\mathcal{Y}\times\mathcal{Y}\to\mathbb{R}$. The maximum mean discrepancy (MMD) is 
\begin{align}
    D_{\mathrm{MMD},k}(P \| Q) \triangleq  \sup_{\|h\|_{\mathcal{H}_k} \le 1} \left| \mathbb{E}_{P}[h(Y)] - \mathbb{E}_{Q}[h(Y)] \right|.
\end{align}
\end{definition}

When the function class $\Phi$ is sufficiently rich (e.g., all bounded continuous functions), the IPM fully characterizes distributional differences. Similarly, when the kernel $k$ is characteristic, the MMD fully characterizes distributional differences. In both cases, $D_{\mathrm{IPM},\Phi}(P \| Q)=0$ (or $D_{\mathrm{MMD},k}(P \| Q)=0$) if and only if $P=Q$.
These quantities can be viewed as special cases of distributional discrepancies defined through restricted function classes, and serve as limiting or simplified alternatives to the $f$-divergence–based bounds considered in the main text.

\paragraph{Problem Statement.} 
Under the assumed DGP in Fig.~\ref{fig:dgp} and Eq.~\eqref{def:scm}, the interventional law $Q_{a,x}$ and a target causal effect in the form of $\mathbb{E}_{Q_{a,x}}[\varphi(Y)]$ (where $\varphi$ is a user-specified and potentially continuous and unbounded function) are generally not identifiable from the observational law $P_{a,x}$ due to unmeasured confounding. To address this, (1) we derive the upper limit of the $f$-divergence of the observational law $P_{a,x}$ from the interventional law $Q_{a,x}$; i.e., $D_{f}(P_{a,x} \| Q_{a,x})$; (2) we translate the upper limit of the $f$-divergence into the sharp interval for causal effects. Throughout the paper, we assume the following:
\begin{assumption}\label{assumption}
    \normalfont
    For all $a,x \in \mathcal{A} \times \mathcal{X}$, 
    \begin{enumerate}[leftmargin=*, itemsep=0pt]
        \item \textbf{Positivity}: $e_a(x) \triangleq \Pr(A=a \mid X=x) \in [c,1-c]$ for some constant $ 0< c  < 1/2$.
        \item \textbf{Mutual absolute continuity}: $P_{a,x} \ll Q_{a, x}$ and $Q_{a, x} \ll P_{a, x}$.
        \item \textbf{Regularity of f}: For the generator function $f$ in the f-divergence, $f(0) < \infty$.   
    \end{enumerate}
\end{assumption}
\section{Divergence Bounds between Observational and Interventional Distributions}\label{sec:3}

We now derive a data-driven upper bound on the $f$-divergence between observational and interventional distributions. Our main result is the following: 
\begin{theorem}[\textbf{f-Divergence Bound}]\label{thm:f-divergence}
    \normalfont
    For any $a \in \mathcal{A}$ and $x \in \mathcal{X}$ such that $P(a \mid x) > 0$, 
    \begin{align}
        D_{f}( P_{a,x} \| Q_{a,x}) \leq B_f(e_a(x)),
    \end{align}
    where 
    \begin{align}
        B_f(e_a(x)) \triangleq e_a(x)f\left(\frac{1}{e_a(x)} \right) + (1-e_a(x))f(0)
    \end{align}
\end{theorem}
Theorem~\ref{thm:f-divergence} establishes that the $f$-divergence $D_f(P_{a,x} \| Q_{a,x})$ is upper bounded by $B_f(e_a(x))$, a function of the propensity score that is directly computable from observational data. Notably, $B_f(e_a(x)) \to 0$ as $e_a(x) \to 1$, since $f$ is continuous (by convexity) and satisfies $f(1) = 0$. Thus, higher propensity scores yield tighter divergence bounds. 

We specialize Thm.~\ref{thm:f-divergence} to standard divergences:
{
\renewcommand{\thecorollary}{T1.1}
\begin{corollary}\label{cor:f-divergence}
    \normalfont
    For any $a \in \mathcal{A}$ and $x \in \mathcal{X}$ such that $P(a \mid x) > 0$, 
    \begin{itemize}[leftmargin=*,itemsep=0pt]
        \item \textbf{KL}: $f(t) \triangleq t \log t$ (with $f(0) = 0$), 
        \begin{align}
            D_{\mathrm{KL}}(P_{a,x} \| Q_{a,x}) \leq -\log e_a(x).   
        \end{align}
        \item \textbf{Hellinger}: $f(t) \triangleq \tfrac{1}{2}(\sqrt{t}-1)^2 $ (with $f(0) = 1/2$), 
        \begin{align}
            D_{\mathrm{H}}(P_{a,x} \| Q_{a,x}) \leq 1- \sqrt{e_a(x)}.
        \end{align}
        \item \textbf{$\chi^2$-divergence}: $f(t) \triangleq \tfrac{1}{2}(t-1)^2 $ (with $f(0) = 1/2$), 
        \begin{align}
            D_{\chi^2}(P_{a,x} \| Q_{a,x}) \leq \frac{1-e_a(x)}{2e_a(x)}.
        \end{align}
        \item \textbf{Total variation}: $f(t) \triangleq \frac{1}{2}|t-1| $ (with $f(0) = \frac{1}{2}$), 
        \begin{align}
            D_{\mathrm{TV}}(P_{a,x} \| Q_{a,x}) \leq 1-e_a(x).
        \end{align}
        \item \textbf{Jensen-Shannon}: $f(t) \triangleq f_{\mathrm{JS}}(t) \triangleq  \frac{1}{2}\left(t\log t - (t+1)\log\!\Big(\tfrac{t+1}{2}\Big)\right)$ (with $f_{\mathrm{JS}}(0)=\frac{1}{2}\log 2$)  
        \begin{align}
            D_{\mathrm{JS}}(P_{a,x} \| Q_{a,x}) \leq B_{f_{\mathrm{JS}}}(e_a(x)) = \frac{1}{2}\log \left( \frac{4 e_a(x)^{e_a(x)}}{(1+e_a(x))^{1+e_a(x)}} \right).
        \end{align}
    \end{itemize}
\end{corollary}
}
Bounds extend to stochastic policies as follows:
{
\renewcommand{\thecorollary}{T1.2}
\begin{corollary}\label{coro:soft}
    \normalfont
    For any stochastic policy $\pi(a \mid x)$, 
    \begin{align}
        D_{f}( P_{\pi} \|  Q_{\pi} ) \triangleq \mathbb{E}_{X}\left[\sum_{a \in \mathcal{A} }\pi(a \mid X) D_{f}( P_{a,X} \|  Q_{a,X} )\right]   \leq \mathbb{E}_{X}\left[ \sum_{a \in \mathcal{A}} \pi(a \mid X)B_f(e_a(X)) \right]. \nonumber
    \end{align}
    Choosing $\pi(a \mid x) = e_a(x)$ yields the global divergence bound: 
    \begin{align}
        D_f(P_{A,X} \,\|\, Q_{A,X})= \mathbb{E}_{X}\!\left[\sum_{a\in\mathcal{A}} e_a(X)\, D_f\!\big(P_{a,X}\,\|\,Q_{a,X}\big)\right] \;\le\; \mathbb{E}_{X}\!\left[\sum_{a\in\mathcal{A}} e_a(X)\, B_f\!\big(e_a(X)\big)\right].\nonumber
    \end{align}
\end{corollary}
}
We derive bounds on the maximum mean discrepancy (MMD; \citealt{gretton2012kernel}), and the integral probability metric (IPM; \citealt{muller1997integral}). 
{
\renewcommand{\thecorollary}{T1.3}
\begin{corollary}[\textbf{IPM and MMD Bounds}]\label{cor:mean-mmd-ipm}
    \normalfont
    Let $\Phi \triangleq \{\varphi: \mathcal{Y} \mapsto [0,1]\}$ be a class of measurable functions. Let $D_{\mathrm{IPM},\mathcal{F}}(P\|Q)$ be the IPM over a function class $\mathcal{F} \triangleq \big\{ f:\ \|f\|_{\infty} < C \big\}$. Let $D_{\mathrm{MMD},\mathtt{k}}(P\|Q)$ be the MMD associated with an RKHS with a kernel $\mathtt{k}$ such that $\mathtt{k}(\cdot,\cdot)<K$.  Then, 
    \begin{align*}
        \textbf{(IPM)} \;&\; D_{\mathrm{IPM},\mathcal{F}_C}\big(P_{a,x}\,\|\,Q_{a,x}\big) \leq 2C\,\min\big\{1-e_a(x),\, \sqrt{-\tfrac{1}{2}\log e_a(x)} \big\}, \\ 
        \textbf{(MMD)} \;&\;  D_{\mathrm{MMD},\mathtt{k}}\big(P_{a,x}\,\|\,Q_{a,x}\big) \leq 2\sqrt{K}\,\min\big\{1-e_a(x),\, \sqrt{-\tfrac{1}{2}\log e_a(x)} \big\}.
    \end{align*}
\end{corollary}
}

All results above extend to the marginal case (without covariates) by setting $X=\emptyset$ and replacing $e_a(x)$ with the marginal propensity score $e_a \triangleq \Pr(A=a)$. This yields bounds on the divergence between the marginal interventional law $Q_a \triangleq P(Y \mid \mathrm{do}(A=a))$ and the marginal observational law $P_a \triangleq P(Y \mid A=a)$.

\subsection{Specialization for Exponential Family}\label{subsec:omit3-exponential-family}
Here, we derive a closed-form $f$-divergence when $P_{a,x},\; Q_{a,x}$ are within the exponential family (e.g., Bernoulli, Gaussian, Poisson, exponential, etc.) to exemplify the mechanism of Thm.~\ref{thm:f-divergence}.
\begin{corollary}[\textbf{Exponential Family}]
    \normalfont
    Suppose $P_{a,x}$ and $Q_{a,x}$ are distributions from a common exponential family:
    \begin{align}
        P_{a,x}(y) &\triangleq  \exp\big(\theta_p^\top T(y) - A(\theta_p)\big)h(y), \\ 
        Q_{a,x}(y) &\triangleq  \exp\big(\theta_q^\top T(y) - A(\theta_q)\big)h(y), 
    \end{align}
    where $\theta_p, \theta_q$ are  natural parameters, $T(y)$ is the sufficient statistics, $A(\theta)$ is the log-partition function (log normalizer), and $h(y)$ is the base measure density. Define $\Delta \triangleq \theta_p - \theta_q$ and $\Delta_A \triangleq A(\theta_p) - A(\theta_q)$. 
    \begin{align}
        D_{f}(P_{a,x}\|Q_{a,x})=
        \mathbb{E}_{Q_{a,x}}\left[
        f\left(
        \exp\big(\Delta^{\intercal} T(Y) - \Delta_A\big)
        \right)
        \right].
    \end{align}
\end{corollary}
\paragraph{Bernoulli Distribution}
Suppose $Y\in\{0,1\}$ and both $P_{a,x}$ and $Q_{a,x}$ are Bernoulli distributions with success probabilities $p$ and $q$, respectively.
The Bernoulli distribution belongs to the exponential family with sufficient statistic $T(y)=y$ and natural parameter $\theta=\log\frac{p}{1-p}$.

The Radon--Nikodym derivative is given by
\begin{align}
    \frac{dP_{a,x}}{dQ_{a,x}}(y) = \exp\left( y \log \frac{p(1-q)}{q(1-p)} + \log \frac{1-p}{1-q} \right).
\end{align}

Consequently, the $f$-divergence admits the representation
\begin{align}
    D_f(P_{a,x}\|Q_{a,x}) = \mathbb{E}_{Q_{a,x}}\!\left[ f\left( \exp\left( Y \log \frac{p(1-q)}{q(1-p)} + \log \frac{1-p}{1-q} \right) \right) \right].
\end{align}

\paragraph{Gaussian Distribution}
Suppose $Y\in\mathbb{R}^d$ and both $P_{a,x}$ and $Q_{a,x}$ are Gaussian distributions with means $\mu_p$, $\mu_q$ and covariance matrices $\Sigma_p$ and $\Sigma_q$, respectively.
In this case, the Gaussian distribution forms an exponential family with sufficient statistic $T(y)=(y, yy^\top)$.

The Radon--Nikodym derivative is given by
\begin{align} \label{eq:RN-gaussian}
    \frac{dP_{a,x}}{dQ_{a,x}}(y) = \frac{|\Sigma_q|^{1/2}}{|\Sigma_p|^{1/2}} \exp\Bigg( -\frac12 (y-\mu_p)^\top \Sigma_p^{-1}(y-\mu_p) + \frac12 (y-\mu_q)^\top \Sigma_q^{-1}(y-\mu_q) \Bigg).
\end{align}

Accordingly, using \eqref{eq:RN-gaussian}, the $f$-divergence can be written as
\begin{align}
    D_f(P_{a,x}\|Q_{a,x}) = \mathbb{E}_{Q_{a,x}}\!\left[ f\!\left( \frac{dP_{a,x}}{dQ_{a,x}}(Y) \right) \right].
\end{align}

\paragraph{Poisson Distribution}

Suppose $Y\in\{0,1,2,\dots\}$ and both $P_{a,x}$ and $Q_{a,x}$ are Poisson distributions with rate parameters $\lambda_p$ and $\lambda_q$, respectively.
The Poisson distribution belongs to the exponential family with sufficient statistic $T(y)=y$ and natural parameter $\theta=\log\lambda$.

The Radon--Nikodym derivative takes the form
\begin{align}
    \frac{dP_{a,x}}{dQ_{a,x}}(y) = \exp\left( y\log\frac{\lambda_p}{\lambda_q} - (\lambda_p-\lambda_q) \right).
\end{align}

The corresponding $f$-divergence is therefore
\begin{align}
    D_f(P_{a,x}\|Q_{a,x}) = \mathbb{E}_{Q_{a,x}}\left[ f\left( \exp\left( Y\log\frac{\lambda_p}{\lambda_q} - (\lambda_p-\lambda_q) \right) \right) \right].
\end{align}

\paragraph{Exponential Distribution}

Suppose $Y\ge0$ and both $P_{a,x}$ and $Q_{a,x}$ follow exponential distributions with rate parameters $\lambda_p$ and $\lambda_q$, respectively.
The exponential distribution is an exponential family with sufficient statistic $T(y)=y$ and natural parameter $\theta=-\lambda$.

For $y\ge0$, the Radon--Nikodym derivative is given by
\begin{align}
    \frac{dP_{a,x}}{dQ_{a,x}}(y) = \frac{\lambda_p}{\lambda_q} \exp\left( -(\lambda_p-\lambda_q)y \right).
\end{align}

Accordingly, the $f$-divergence can be expressed as
\begin{align}
    D_f(P_{a,x}\|Q_{a,x}) = \mathbb{E}_{Q_{a,x}}\left[ f\left( \frac{\lambda_p}{\lambda_q} \exp\left( -(\lambda_p-\lambda_q)Y \right) \right) \right].
\end{align}


\section{A Distributionally Robust Formulation of Causal Bounds} \label{sec:4}
In this section, we leverage the upper bounds on statistical divergence derived in Section~\ref{sec:3} to construct bounds on the target causal effect $\theta(a,x) \triangleq \mathbb{E}_{Q_{a,x}}[\varphi(Y)]$, where $\varphi(Y)$ is an arbitrary measurable function with finite first and second moments. This framework encompasses diverse causal quantities: setting $\varphi(Y) \triangleq \boldsymbol{1}(Y \leq t)$ yields the cumulative distribution function $Q_{a,x}(Y \leq t)$, while choosing $\varphi(Y) \triangleq \ell(Y;\theta)$ (a loss function for $\theta$) yields the risk function over $Q_{a,x}$. Crucially, we impose no restrictions requiring $\varphi$ to be discrete or bounded.

Using the divergence bound $D_f(P_{a,x} \| Q_{a,x}) \leq B_f(e_a(x))$ from Thm.~\ref{thm:f-divergence}, we define the f-divergence-based \textit{ambiguity set}, which is a collection of distributions over $Y$ within the $B_f(e_a(x))$ radius around the observational law $P_{a,x}$: 
\begin{align}\label{eq:ambiguity-set}
\textbf{(Ambiguity set)}\;\;
\mathcal Q_f(a,x; P_{a,x}) \triangleq 
\left\{
\begin{aligned}
Q_{a,x}\in\mathcal P_{a,x}(\mathcal Y)
\;\; :\;\;
& D_f\!\left(P_{a,x}\,\|\, Q_{a,x}\right) \le B_f(e_a(x)), \\
& P_{a,x} \ll Q_{a,x}
\end{aligned}
\right\},
\end{align}
where $\mathcal{P}_{a,x}(Y)$ is a collection of probability laws given $A=a$ and $X=x$. The target causal effect $\mathbb{E}_{Q_{a,x}}[\varphi(Y)]$ is bounded by expectations over the extremal distributions in this ambiguity set:
\begin{align}\label{eq:initial-optimization}
    \textbf{(Bounds)}\quad \underbrace{ \theta_{\mathrm{lo}}(a,x) }_{\inf_{Q \in \mathcal Q_f(a,x)} \mathbb{E}_{Q}[\varphi(Y)]} \leq \underbrace{\theta(a,x)}_{\mathbb{E}_{Q_{a,x}}[\varphi(Y)]} \leq \underbrace{ \theta_{\mathrm{up}}(a,x) }_{ \sup_{Q \in \mathcal Q_f(a,x)} \mathbb{E}_{Q}[\varphi(Y)] }
\end{align}
The lower and upper bounds are symmetric: by Proposition~\ref{prop:lower-upper} below, the lower bound can be obtained from the upper bound by negating the function $\varphi$. Therefore, we focus on deriving the upper bound $\theta_{\mathrm{up}}(a,x)$ without loss of generality. 
\begin{proposition}[\textbf{Lower bound as a subproblem of upper bound}]\label{prop:lower-upper}
    \normalfont
    Let 
    \begin{align}
        \theta_{\mathrm{lo}}(a,x; \varphi) \triangleq \inf_{Q \in \mathcal Q_f(a,x)} \mathbb{E}_{Q}[\varphi(Y)], \qquad \theta_{\mathrm{up}}(a,x;\varphi) \triangleq \sup_{Q \in \mathcal Q_f(a,x)} \mathbb{E}_{Q}[\varphi(Y)].
    \end{align}
    Then, 
    \begin{align}
        \theta_{\mathrm{lo}}(a,x; \varphi) = -  \theta_{\mathrm{up}}(a,x;-\varphi).
    \end{align}
\end{proposition}

By Proposition~\ref{prop:lower-upper}, it suffices to compute $\theta_{\mathrm{up}}(a,x)$. However, computing $\theta_{\mathrm{up}}(a,x)$ directly from Eq.~\eqref{eq:initial-optimization} is intractable, as it requires optimizing over the infinite-dimensional space of all probability measures in $\mathcal{Q}_f(a,x)$. To overcome this computational barrier, we reformulate the problem using convex duality:
\begin{theorem}[\textbf{Primal and Dual Formulations}]\label{thm:representation-bound}
    \normalfont
    Let $s(Y) \triangleq \tfrac{dQ_{a,x}}{dP_{a,x}}(Y)$ denote the likelihood ratio, $g_s(Y) \triangleq s(Y)\cdot f(1/s(Y))$, and $\eta_f(a,x) \triangleq B_f(e_a(x))$. The upper bound $\theta_{\mathrm{up}}(a,x)$ admits the following equivalent representations: 
    \begin{align}
        \theta_{\mathrm{up}}(a,x) & = \sup_{s > 0}\Big\{ \mathbb{E}_{P_{a,x}}[s(Y) \varphi(Y)]\; \text{ s.t. } \mathbb{E}_{P_{a,x}}[s(Y)] = 1, \;  \mathbb{E}_{P_{a,x}}\big[g_s(Y)] \leq \eta_f(a,x)   \Big\} \label{eq:primal} \\ 
                &= \inf_{\lambda > 0, u \in \mathbb{R}}\big\{ \lambda \eta_f(a,x)  + u +  \lambda\mathbb{E}_{P_{a,x}}\Big[ g^{\ast}\big(\tfrac{\varphi(Y) - u}{\lambda}\big)\Big] \big\}, \label{eq:dual}
    \end{align}
    where $g^\ast(t) \triangleq \sup_{s > 0}\{s\,t - g(s) \}$ is the convex conjugate (also known as the Legendre–Fenchel conjugate or c-transform) of $g$.
\end{theorem}
The following proposition provides a general recipe for computing the convex conjugate $g^{\ast}$:

\begin{proposition}[\textbf{Convex Conjugate} $g^{\ast}$]\label{prop:convex-conjugate-unified}
    \normalfont
    Let $f:(0,\infty)\to (-\infty,\infty]$ be proper, convex, and lower semi-continuous function. Define for $s > 0$, 
    \begin{align}
        g(s) \triangleq s f(1/s), \quad g^{\ast}(t) \triangleq \sup_{s>0}\{st - g(s)\}. 
    \end{align}
    Let $r \triangleq 1/s$. Then, 
    \begin{align}
        g^{\ast}(t) \triangleq \sup_{r > 0} \frac{t-f(r)}{r}.
    \end{align}
    Moreover, if the supremum is attained at some $r^{\ast} > 0$, then there exists a subgradient $a \in \partial f(r^{\ast})$ such that 
    \begin{align}
        t = f(r^{\ast}) - r^{\ast}a, \quad \text{ and } \quad g^{\ast}(t) = -a. 
    \end{align}
    If $f$ is differentiable at $r^{\ast}$, then $a = f'(r^{\ast})$ and hence $g^{\ast}(t) = -f'(r^{\ast})$. 
\end{proposition}

Proposition~\ref{prop:convex-conjugate-unified} provides a constructive method for evaluating the convex conjugate $g^{\ast}$. By introducing the change of variable $r = 1/s$, we transform the optimization problem into a simpler form. The result states that the value of the conjugate function $g^\ast(t)$ is determined by the derivative (or subgradient) of the original divergence function $f$ at the optimal point $r^\ast$.

We apply Prop.~\ref{prop:convex-conjugate-unified} to standard f-divergences:
{
\renewcommand{\thecorollary}{P1}
\begin{corollary}\label{cor:convex-conjugate}
    \normalfont
    Let $g(s) \triangleq sf(1/s)$ for $s > 0$. Then, 
    \begin{itemize}[leftmargin=*]
        \item \textbf{KL}: $g_{\mathrm{KL}}(s) = -\log s$, and 
        \begin{align}
            g^{\ast}_{\mathrm{KL}}(t) = 
            \begin{cases}
                 -1-\log(-t) & \text{ if } t < 0; \\ 
                 +\infty &\text{ if } t \geq 0.
            \end{cases}
        \end{align}
        \item \textbf{Hellinger}: $g_{\mathrm{H}}(s) = \frac{1}{2}(1-2\sqrt{s} + s)$, and 
        \begin{align}
            g^{\ast}_{\mathrm{H}}(t) = 
            \begin{cases}
                 \frac{t}{1-2t} & \text{ if } t < 1/2; \\ 
                 +\infty &\text{ if } t \geq 1/2.
            \end{cases}
        \end{align}
        \item \textbf{Chi-square}: $g_{\chi^2}(s) = \frac{(1-s)^2}{2s}$, and  
        \begin{align}
            g^{\ast}_{\chi^2}(t)  = 
            \begin{cases}
                 1-\sqrt{1-2t} & \text{ if } t \leq 1/2; \\ 
                 +\infty &\text{ if } t > 1/2.
            \end{cases}
        \end{align}
        \item \textbf{TV}: $g_{\mathrm{TV}}(s) = \frac{1}{2}|1-s|$, and 
        \begin{align}
            g^{\ast}_{\mathrm{TV}}(t)  = 
            \begin{cases}
                 -\frac{1}{2}, & \text{ if } t \leq -\frac{1}{2}, \\
                 t, & \text{ if } -\frac{1}{2} < t \leq \frac{1}{2}, \\
                 +\infty, &\text{ if } t > \frac{1}{2}. 
            \end{cases}
        \end{align}
        \item \textbf{Jensen-Shannon}: $g_{\mathrm{JS}}(s) = \frac{1}{2}\left(s\log s - (1+s) \log(1+s) + (1+s)\log 2\right)$, and 
        \begin{align}
            g^{\ast}_{\mathrm{JS}}(t)  = 
            \begin{cases}
                 -\frac{1}{2}\log \left( 2 - \exp(2t) \right), & \text{ if } t < \frac{1}{2}\log 2, \\
                 +\infty, &\text{ if } t \geq \frac{1}{2}\log 2. 
            \end{cases}
        \end{align}
    \end{itemize}
\end{corollary}
}

All results above extend to the marginal case (without covariates) by setting $X=\emptyset$ and replacing $e_a(x)$ with the marginal propensity score $e_a \triangleq \Pr(A=a)$. This yields bounds on the marginal causal effect $\mathbb{E}_{Q_a}[Y] \triangleq \mathbb{E}[Y \mid \mathrm{do}(A=a)]$, where $Q_a \triangleq P(Y \mid \mathrm{do}(A=a))$. 

\section{Debiased Semiparametric Estimation of Causal Bounds}\label{sec:5}
Solving the dual problem in Eq.~\eqref{eq:dual} pointwise for each $(a,x)$ is computationally intractable, as it requires estimating the conditional expectation $\mathbb{E}_{P_{a,x}}[g^{\ast}(\cdot)]$ separately for each pair of covariate $X=x$ and treatment $A=a$ at every optimization iteration. We circumvent this by amortizing the optimization as follows: we view $\lambda(a,x)$ and $u(a,x)$ as functional parameters to be learned globally. Parameterizing $\lambda(a,x) \triangleq \exp(h(a,x))$ to enforce positivity, the dual problem transforms into:
\begin{proposition}
    \normalfont 
    Let $\eta_f(a,x) \triangleq B_f(e_a(x))$. Then, 
    \begin{align}\label{eq:dual2}
        \theta_{\mathrm{up}}(a,x) = \inf_{h(a,x) \in \mathbb{R} \atop u(a,x) \in \mathbb{R}} \mathbb{E}_{P_{a,x}}\Big[ \exp({h(A,X)})\big\{\eta_f(A,X) + g^{\ast}\big( \tfrac{\varphi(Y) - u(A,X)}{\exp({h(A,X)})}\big) \big\}  + u(A,X)  \Big].
    \end{align}
\end{proposition}

To operationalize this optimization, we define a loss function and corresponding risk function for the functional parameters $h$ and $u$:
\begin{definition}[\textbf{Risk Function for Causal Bound}]\label{def:risk}
    \normalfont
    Let $V = (X,A,Y)$. Let $h_{\beta}, u_{\gamma}: \mathcal{A} \times \mathcal{X} \mapsto \mathbb{R}$ be maps parametrized by $\beta \in \mathbb{R}^{p_1}$ and $\gamma \in \mathbb{R}^{p_2}$. 
    The risk function for causal bounds is 
    \begin{align}\label{eq:risk}
        \mathcal{R}(\beta,\gamma; e) \triangleq \mathbb{E}_{P}[\ell(V; (\beta, \gamma), e )],
    \end{align}
    where $e \triangleq e_A(X)$ and $\eta_f \triangleq \eta_f(A,X) \triangleq B_f(e_A(X))$, and 
    \begin{align}\label{eq:loss}
        \ell(V; (\beta, \gamma), e ) \triangleq \exp(h_{\beta}(A,X))\Big\{\eta_f(A,X) + g^{\ast}\Big( \tfrac{\varphi(Y) - u_{\gamma}(A,X)}{\exp(h_{\beta}(A,X))}\Big) \Big\}  + u_{\gamma}(A,X).
    \end{align}
\end{definition}
The following proposition shows that this risk minimization is equivalent to solving the pointwise dual problem:
\begin{proposition}[\textbf{Justification of Risk Function}]\label{prop:justification}
    \normalfont
    Define, for each $(a,x)$, the following loss
    \begin{align}
        \ell(h, u; y,a,x) \triangleq \exp(h(a,x))\Big\{\eta_f(a,x) + g^{\ast}\big( \tfrac{\varphi(y) - u(a,x)}{\exp(h(a,x))}\big) \Big\}  + u(a,x). 
    \end{align}
    Let $\mathcal{R}(h,u) \triangleq \mathbb{E}_{P}[\ell(h, u; Y,A,X)]$ be a risk function. Assume $\mathcal{R}(h,u) < \infty$ for all $h,u \in \mathcal{F}$, where $\mathcal{F}$ is a function class rich enough that for any $(h_1,u_1), (h_2,u_2) \in \mathcal{F}$ and $\forall B \subset \mathcal{A} \times \mathcal{X}$, $(h',u') \triangleq (h_1,u_1)\boldsymbol{1}_{B} + (h_2,u_2)\boldsymbol{1}_{B^c}$ also lies in $\mathcal{F}$. Then, for any fixed $(h^{\star}, u^{\star}) \in \mathcal{F}$, the followings are equivalent:
    \begin{enumerate}[leftmargin=*]
        \item $(h^{\star}, u^{\star})$ minimizes $\mathcal{R}$ over $\mathcal{F}$.
        \item $(h^{\star}, u^{\star})$ minimizes $\mathbb{E}_{P_{a,x}}[\ell(h,u;Y,a,x)]$ for $P_{A,X}$-almost every $(a,x)$. 
    \end{enumerate}
\end{proposition}
Proposition~\ref{prop:justification} establishes that solving the pointwise dual problem in Eq.~\eqref{eq:dual2} is equivalent to finding the global minimizer $(h^{\star}, u^{\star})$ of the risk function in Def.~\ref{def:risk}. This amortization substantially improves tractability: instead of solving a separate optimization for each $(a,x)$, we learn functional parameters that generalize across the covariate space.

Since the risk function in Eq.~\eqref{eq:risk} depends on the unknown propensity score $e$, we must estimate it from data. However, estimating $e$ introduces errors that can propagate into the bound estimates. To mitigate this, we construct a debiased risk function that achieves first-order insensitivity (Neyman-orthogonality) to perturbations in $e$:
\begin{definition}[\textbf{Debiased Risk Function}]\label{def:debiased-risk}
    \normalfont
    Let $\eta'_f(A,X)$ be the first-order derivative of $\eta_f(A,X)$ w.r.t. $e$. The debiased risk function is 
    \begin{align}\label{eq:debiased-risk}
        \mathcal{R}^{\mathrm{db}}(\beta, \gamma; e) \triangleq \mathbb{E}[\ell^{\mathrm{db}}(V; (\beta, \gamma), e)],
    \end{align}
    where 
    \begin{align}
        \ell^{\mathrm{db}}(V; (\beta, \gamma), e) &\triangleq  \underbrace{ \exp(h_{\beta}(A,X))\Big\{\eta_{f}(A,X) + g^{\ast}\big( \tfrac{\varphi(Y) - u_{\gamma}(A,X)}{\exp(h_{\beta}(A,X))}\big) \Big\}  + u_{\gamma}(A,X) }_{\text{Eq.~\eqref{eq:loss}}} \\ 
        &+ \sum_{a\in\mathcal{A}} e_a(X)\exp(h_{\beta}(a,X)) \eta'_{f}(a,X)\,\big(\boldsymbol{1}(A=a) - e_a(X)\big) \label{eq:debiasedness}
    \end{align}
\end{definition}
Here, Eq.~\eqref{eq:debiasedness} is an error correction term, which makes $\ell^{\mathrm{db}}(V; (\beta, \gamma), e)$ invariant to small first-order perturbations in $e$ (i.e., \textit{Neyman-orthogonal} \citep{chernozhukov2018double}):
\begin{lemma}[\textbf{Orthogonality}]\label{lemma:orthogonality}
    \normalfont
    For any direction functions $\{s_a(\cdot)\}_{a \in \mathcal{A}}$ and any perturbation path $e_{t,a} \triangleq e_a + t s_a$ with sufficiently small $|t|$, $\tfrac{\partial }{\partial t}\mathcal{R}^{\mathrm{db}}(\beta,\gamma;e_t)\big\vert_{t=0} = 0$ for all $(\beta,\gamma)$. 
\end{lemma}
We now present our estimation procedure based on cross-fitting:
\begin{definition}[\textbf{Debiased Causal Bound Estimators}]\label{def:cross-fit}
    \normalfont
    Fix a functional $\varphi$ and an $f$-divergence. Let $\ell^{\mathrm{db}}$ and $\mathcal{R}^{\mathrm{db}}$ be as in Def.~\ref{def:debiased-risk}. The debiased estimator of the upper causal bound $\theta_{u}(a,x)$ for any $(a,x) \in \mathcal{A} \times \mathcal{X}$ is constructed as follows:
    \begin{enumerate}[leftmargin=*]
        \item Randomly split the dataset $\mathcal{D}$ (with size $n$) into $K$ disjoint folds  $\mathcal{D}_1,\cdots,\mathcal{D}_K$. 
        \item For each $k$ fold, learn $\widehat{e}^{k}_a$ using $\mathcal{D}_{-k} \triangleq \mathcal{D} \setminus \mathcal{D}_{k}$ for all $a \in \mathcal{A}$.
        \item For each fold $k$, solve $\widehat{\vartheta}_k \triangleq (\hat \beta_k,\hat \gamma_k) \in \arg\min_{\beta, \gamma} \sum_{i \mid V_i \in \mathcal{D}_k} \ell^{\mathrm{db}}(V_i; (\beta,\gamma),\widehat{e}^k)$. 
        \item For each fold $k$, evaluate 
        \begin{align}
            & \widehat h_k(a,x) \triangleq h_{\widehat\beta_k}(a,x),\qquad \widehat u_k(a,x) \triangleq u_{\widehat\gamma_k}(a,x),\\ 
            &\widehat\lambda_k(a,x) \triangleq \exp\{\widehat h_k(a,x)\}, \qquad  \widehat\eta^{\,k}_f(a,x) \triangleq B_f\big(\widehat e^{\,k}_a(x)).
        \end{align}
        \item For each fold $k$ and each $i \in \mathcal{D}_k$, evaluate $Z^{k}_i \triangleq g^{\ast}\big(\tfrac{\varphi(Y_i) - \widehat u_k(A_i,X_i)}{\widehat\lambda_k(A_i,X_i)}\big)$, and learn a regressor $\widehat{m}_{k} $ by regressing $Z^{k}_i$ onto $(A,X)$ using $\mathcal{D}_k$. 
        \item Evaluate $\widehat{\theta}^{(k)}_{\mathrm{up}}(a,x) \triangleq \widehat\lambda_k(a,x)\big(
        \widehat\eta^{\,k}_f(a,x) + \widehat m_k(a,x) \big) + \widehat u_k(a,x)$ and return $\widehat{\theta}_{\mathrm{up}}(a,x)  \triangleq (1/K)\sum_{k=1}^{K}\widehat{\theta}^{(k)}_{\mathrm{up}}(a,x) $.
    \end{enumerate}
\end{definition}

We now analyze the error of the proposed debiased estimator under following set of assumptions: 
\begin{assumption}[\textbf{Regularity-1}]\label{assumption:regularity-condition}
    \normalfont
    Let $e \triangleq \{e_a(\cdot): a \in \mathcal{A}\}$ be the true propensity score, $\vartheta \triangleq (\beta,\gamma)$ and $\vartheta_0 \triangleq (\beta_0,\gamma_0) \in \arg\min_{\beta,\gamma} \mathcal{R}^{\mathrm{db}}(\beta,\gamma; e)$. 
    \begin{enumerate}[leftmargin=*]
        \item \textbf{Positivity:} $e_a(x) \in [c,1-c]$ for some constant $ 0< c  < 1/2$ for all $a,x \in \mathcal{A} \times \mathcal{X}$. 
        \item \textbf{f-divergence regularity:} $f$ is convex and twice continuously differentiable; and the induced radius $B_f(e_a(x))$ is twice continuously differentiable on $[c,1-c]$, with bounded first and second derivative; i.e., $\sup_{e \in [c,1-c]} |B_f(e)| + |B'_f(e)| + |B''_{f}(e)| < \infty $. 
        \item \textbf{Loss regularity:} For each fixed $e \in [c,1-c]$, the map $\vartheta \mapsto \ell^{\mathrm{db}}(V;\vartheta,e)$ is twice continuously differentiable, with 
        \begin{align*}
            \sup_{\vartheta,e} \|\ell^{\mathrm{db}}(V;\vartheta,e)\|^2_2 < \infty, \;\;\; \sup_{\vartheta,e} \|\nabla_\theta \ell^{\mathrm{db}}(V;\vartheta,e)\|^2_2 < \infty, \;\;\; \sup_{\theta,e}\|\nabla^2_{\vartheta\vartheta}\ell^{\mathrm{db}}(V;\vartheta,e)\|^2_2 < \infty.
        \end{align*}
        \item \textbf{Higher-order smoothness:} Let $H(\vartheta; e) \triangleq \nabla^{2}_{\vartheta\vartheta}R^{\mathrm{db}}(\vartheta; e)$. There exists a neighborhood $\varTheta_0$ of $\vartheta$ containing $\vartheta_0$ and constants $0 < \kappa \leq \kappa_2 < \infty$ such that 
        \begin{align}
            \kappa_1 \mathbf{I} \;\preceq\; H(\vartheta; e) \;\preceq\; \kappa_2 \mathbf{I} \quad\text{for all }\vartheta\in\varTheta_0.
        \end{align}
        \item \textbf{Uniform LLN:} For each fold $k$, define the empirical risk w.r.t. $\ell^{\mathrm{db}}$ with the training fold is $\widehat{R}^{\mathrm{db}}_k(\vartheta, \widehat{e}^k)$. Then, we have a uniform law-of-large-number:
        \begin{align}
            \sup_{\vartheta} \big| \widehat R_k^{\mathrm{db}}(\vartheta; \widehat e^{\,k}) - R^{\mathrm{db}}(\vartheta; \widehat e^{\,k}) \big| = O_p(n^{-1/2}),
        \end{align}
    \end{enumerate}
\end{assumption}

\begin{assumption}[\textbf{Regularity-2}]\label{assumption:regularity-condition-2}
    \normalfont
    For $\vartheta \triangleq (\beta, \gamma)$, let $Z_{\vartheta} \triangleq g^{\ast}\big( \tfrac{\varphi(Y) - u_{\gamma}(A,X)\}}{ \exp(h_\beta(A,X))}\big)$ and $m_{\vartheta}(a,x) \triangleq \mathbb{E}[Z_{\vartheta} \mid A=a,X=x]$. Let $\widehat{m}_k$ be the estimate for $m_{\vartheta}$ using the $k$-fold data.
    \begin{enumerate}[leftmargin=*]
        \item \textbf{Bounded nuisances:} $h_{\beta_0}, u_{\gamma_0}, h_{\hat{\beta}_k}, u_{\hat{\gamma}_k}$ are bounded by some constant $M$. 
        \item \textbf{Lipschitz parameterization:} The map $(\beta,\gamma) \mapsto (h_{\beta}(a,x), u_{\gamma}(a,x)) $ is Lipschitz in $\vartheta = (\beta,\gamma)$ uniformly over all $(a,x)$ with constant $L_{\vartheta}$; i.e., 
        \begin{align}
            \sup_{a,x}\big(|h_{\beta}(a,x)-h_{\beta'}(a,x)| + |u_{\gamma}(a,x)-u_{\gamma'}(a,x)| \big) \le L_\vartheta\|\vartheta-\vartheta'\|.
        \end{align}
        \item \textbf{Smoothness of $g^{\ast}$:} The convex conjugate $g^{\ast}$ is continuously differentiable with bounded derivative; i.e., $\sup_{t \in \mathcal{T}}|(g^{\ast})'(t)| < \infty$ where $\mathcal{T}$ is a range where $g^{\ast}(t)$ is well-defined. 
        \item \textbf{Assumption on regression:} $\|\widehat m_k - m_{\widehat\vartheta_k}\|_2 = O_p(s_n)$, where  $s_n$ is some sequence $s_n \rightarrow 0$; There exists a constant $L_m$ s.t. $\|m_{\vartheta} - m_{\vartheta'}\|_2 \leq L_m\|\vartheta-\vartheta'\|$. 
        \item \textbf{Correct model choice}: Let $\overline{\theta}^{\star}_{\varphi,0}(a,x) \triangleq \mathbb{E}[\ell(V; (\beta_0,\gamma_0),e) \mid A=a,X=x]$. Then $\overline{\theta}^{\star}_{\varphi,0}(a,x) \triangleq \mathbb{E}[\ell(V; (\beta_0,\gamma_0),e) \mid A=a,X=x] = \overline{\theta}_{\varphi}(a,x)$ for all $(a,x)$. 
    \end{enumerate}
\end{assumption}

We now formalize the convergence rate of the proposed debiased estimator:
\begin{theorem}[\textbf{Error Analysis}]\label{thm:error-analysis}
    \normalfont
    Under Assumption~\ref{assumption:regularity-condition}, fix a fold $k$. Let $e$ be the true propensity score, and $\vartheta_0 \triangleq (\beta_0,\gamma_0) \in \arg\min_{\vartheta} \mathcal{R}^{\mathrm{db}}(\vartheta; e)$ for $\vartheta \triangleq (\beta,\gamma)$. Let $\widehat{\vartheta}_k \triangleq (\widehat{\beta}_k, \widehat{\gamma}_k)$ be the minimizer from Step 3 in Def.~\ref{def:cross-fit} with $\widehat{e}^k$. Define $r_n \triangleq O_p(\| \widehat{e}^k - e \|_2) $. Then, 
    \begin{align}\label{eq:thm:error-analysis-1}
        \|\widehat\vartheta_k - \vartheta_0\|^2_{2} = O_p\big(n^{-1/2} + r_n^2\big).
    \end{align}
    Furthermore, let $Z_{\vartheta} \triangleq g^{\ast}\big( \tfrac{\varphi(Y) - u_{\gamma}(A,X)}{ \exp(h_\beta(A,X))}\big)$ and $m_{\vartheta}(a,x) \triangleq \mathbb{E}[Z_{\vartheta} \mid A=a,X=x]$. Define $s_n \triangleq O_p(\|\widehat{m}_k - m_{\widehat{\vartheta}_k} \|_2)$ where $\widehat{m}_k$ is from Step 5 in Def.~\ref{def:cross-fit}. Let $\widehat{\theta}^{(k)}_{\mathrm{up}}$ be the estimated upper causal bound for the fold $k$. Under additional Assumption~\ref{assumption:regularity-condition-2},
    \begin{align}\label{eq:thm:error-analysis-2}
        \big\| \widehat{\theta}^{(k)}_{\mathrm{up}} - \theta_{\mathrm{up}} \big\|_2^2 = O_p\big( n^{-1/2} + r_n^2 + s_n^2 \big).
    \end{align}
\end{theorem}
Thm.~\ref{thm:error-analysis} demonstrates the sample efficiency of our debiased estimator. Even when the nuisance components (the propensity score and the pseudo-outcome regression) converge slowly (e.g., at rate $n^{-1/4}$), both the dual parameters $\widehat{\vartheta}_k$ and the upper-bound estimator $\widehat{\theta}_{\mathrm{up}}^{(k)}$ achieve the faster rate (e.g., at rate $n^{-1/2}$). Specifically, the sample-efficiency gain is of order $O_p(r^2_n)$ rather than $O_p(r_n)$ that would result from using the non-debiased risk function (Eq.~\eqref{eq:risk}). This improvement stems from the orthogonal construction of the debiased risk, which eliminates first-order sensitivity to propensity score errors (Lemma~\ref{lemma:orthogonality}). Consequently, nuisance components can be estimated using flexible machine learning methods while the estimator retains faster convergence rates.

\subsection{Ensemble Bound Aggregation}\label{sec:5.1}
Different $f$-divergences encode distinct notions of distributional discrepancy, and no single divergence uniformly dominates others in tightness across all data distributions. Consequently, we estimate bounds using a collection $\mathcal{F}$ of $f$-generators (e.g., $\mathcal{F} = \{f_{\mathrm{KL}}, f_{\mathrm{TV}}, f_{\chi^2}, \cdots \}$), yielding a family of upper bound estimates $\widehat{\boldsymbol{\theta}}_{\mathrm{up}} \triangleq \{\widehat{\theta}_{\mathrm{up},f}: f \in \mathcal{F}\}$ where $\widehat{\theta}_{\mathrm{up},f}$ is an upper-bound estimate for a fixed $f \in \mathcal{F}$. Let $\widehat{\boldsymbol{\theta}}_{\mathrm{lo}}$ be defined similarly. To construct the tightest valid interval, we aggregate these bounds while accounting for potential finite-sample violations due to estimation error and numerical instability. 

Our aggregation strategy addresses this challenge via order statistics:
\begin{definition}[\textbf{$k$-th order statistics aggregator}]\label{def:kth}
\normalfont
    Let $\widehat{\boldsymbol{\theta}}_{\mathrm{lo}}, \widehat{\boldsymbol{\theta}}_{\mathrm{up}}$ 
    denote candidate lower and upper bounds, respectively, with  $n_f \triangleq \lvert \widehat{\boldsymbol{\theta}}_{\mathrm{lo}} \rvert  = \lvert \widehat{\boldsymbol{\theta}}_{\mathrm{up}} \rvert$. 
    For $k \in \{1,\dots,n_f\}$, the \emph{$k$-th order-statistics aggregator} ($k$-agg) is defined as the pair $(\widehat{\theta}_{\mathrm{lo}}^{k}, \widehat{\theta}_{\mathrm{up}}^{k})$, where $\widehat{\theta}_{\mathrm{lo}}^{k}$ is the $k$-th largest element of  $\widehat{\boldsymbol{\theta}}_{\mathrm{lo}}$ and $\widehat{\theta}_{\mathrm{up}}^{k}$ is the $k$-th smallest element of $\widehat{\boldsymbol{\theta}}_{\mathrm{up}}$.
\end{definition}

The following lemma formalizes the validity condition for the $k$-th order aggregator:

\begin{lemma}[\textbf{Valid Coverage under Partial Correctness}]\label{lemma:valid-coverage}
    \normalfont
    For a fixed $(a,x)$, 
    \begin{itemize}[leftmargin=*]
        \item $\widehat{\theta}^{k}_{\mathrm{lo}}(a,x) \le \theta(a,x)$ iff at least $(n_f - k + 1)$ elements of $\widehat{\boldsymbol{\theta}}_{\mathrm{lo}}$ are smaller or equal to $\theta(a,x)$. 
        \item $\widehat{\theta}^{k}_{\mathrm{up}}(a,x) \ge \theta(a,x)$ iff at least $(n_f - k + 1)$ elements of $\widehat{\boldsymbol{\theta}}_{\mathrm{up}}$ are greater or equal to $\theta(a,x)$. 
    \end{itemize}
\end{lemma}
Lemma~\ref{lemma:valid-coverage} guarantees that the $k$-agg produces valid bounds as long as at least $(n_f - k + 1)$ divergences yield correct estimates. This robustness property is critical: even if a minority of divergences fail (due to finite-sample violations or numerical issues), the aggregator automatically discards outliers by selecting the $k$-th order statistic. In practice, the $k$-agg is implemented by initializing $k = 1$ (selecting the tightest bounds) and iteratively incrementing $k \leftarrow k+1 $ until $\widehat{\theta}^{k}_\mathrm{lo} \leq \widehat{\theta}^{k}_{\mathrm{up}}$ is satisfied.

\subsection{Debiased Estimation for Average Causal Effects}
When covariates are absent ($X = \emptyset$), the estimation procedure simplifies substantially. The marginal propensity score $e_a \triangleq \Pr(A=a)$ can be estimated at rate $o_P(n^{-1/2})$ via sample proportions, eliminating the need for the debiasing correction in Eq.~\eqref{eq:debiasedness}. We now specialize our framework to this covariate-free setting.
\begin{definition}[\textbf{Risk Function (Marginal Case)}]\label{def:risk-emptyX}
    \normalfont
    Let $h \triangleq \{h_a \in \mathbb{R}^{+}: a \in \mathcal{A}\}$ and $u \triangleq \{u_a \in \mathbb{R}^{+}: a \in \mathcal{A}\}$. Let $V \triangleq (A,Y)$ and $\eta^a_f \triangleq B_f(e_a)$. A risk function for causal bound when $X=\emptyset$ is 
    \begin{align}\label{eq:risk-ate}
        \mathcal{R}(h,u;e) \triangleq \mathbb{E}_{P}[\ell(V;(h,u),\eta_f)],
    \end{align}
    where 
    \begin{align}\label{eq:loss-ate}
        \ell(V; (h,u), e ) \triangleq \exp({h_A})\big\{\eta^a_f + g^{\star}\big( \tfrac{\varphi(Y) - u_{A}}{\exp({h_A})} \big) \big\} + u_{A}. 
    \end{align}
\end{definition}
The estimator for the marginal case directly minimizes the risk in Def.~\ref{def:risk-emptyX} without debiasing: 
\begin{definition}[\textbf{Bound Estimator (Marginal Case)}]\label{def:bound-average}
    \normalfont
    Fix a functional $\varphi$ and an $f$-divergence. Let $\ell$ and $\mathcal{R}$ be as in Def.~\ref{def:risk-emptyX}. Let the observed sample be i.i.d. $\{V_i \triangleq (A_i, Y_i)\}_{i=1}^{n}$. Define $n_a \triangleq \sum_{i=1}^{n}\boldsymbol{1}(A_i = a)$. The estimator of the upper causal bound $\overline{\theta}_{\varphi} (a)$ for any $a \in \mathcal{A}$ is constructed as follows:
    \begin{enumerate}[leftmargin=*]
        \item Estimate the marginal propensity $\widehat{e}_a \triangleq n_a /n $.
        \item Solve $\widehat{\vartheta} \triangleq (\widehat{h},\widehat{u}) \in \arg\min_{h,u}\sum_{i=1}^{n}\ell(V_i; (h,u), \widehat{e})$. 
        \item Evaluate $\widehat{\lambda}_a \triangleq \exp(\widehat{h}_a)$. 
        \item Define the pseudo-outcome $\widehat Z_i \equiv g^{\ast}\!\Big(\tfrac{\varphi(Y_i)-\widehat {u}_{A_i}}{\widehat\lambda_{A_i}}\Big)$ and evaluate $\widehat{m}_a \triangleq (1/n_a)\sum_{i:A_i = a}\widehat{Z}_i$. 
        \item Return $\widehat\theta_{\varphi,f}(a)
\equiv \widehat\lambda_a\Big(\widehat\eta_{f,a}+\widehat m_a\Big)+\widehat u_a$, for $a \in \mathcal{A}$. 
    \end{enumerate}
\end{definition}

We now analyze the error of the proposed debiased estimator under following set of assumptions: 
\begin{assumption}[\textbf{Regularity (Marginal Case)}]\label{assumption:regularity-condition-3}
    \normalfont
    Let $e \triangleq \{e_a: a \in \mathcal{A}\}$ where $e_a \triangleq \Pr(A=a)$, $\vartheta \triangleq (\beta,\gamma)$ and $\vartheta_0 \in \arg\min_{\vartheta} \mathcal{R}(\vartheta; e)$ where $\vartheta \triangleq (h,u) \triangleq  \{(h_a,u_a): a \in \mathcal{A}\}$. Let $Z_{\vartheta} \triangleq g^{\ast}\big( \tfrac{\varphi(Y) - u_A}{\exp(h_A)}\big)$. Let $m_{\vartheta,a} \triangleq \mathbb{E}_{P_a}[Z_{\vartheta}]$. 
    \begin{enumerate}[leftmargin=*]
        \item \textbf{Positivity:} $e_a \in [c,1-c]$ for some constant $ 0< c  < 1/2$ for all $a \in \mathcal{A} $. 
        \item \textbf{f-divergence regularity:} $f$ is convex and twice continuously differentiable; and the induced radius $B_f$ is twice continuously differentiable on $[c,1-c]$ with bounded derivatives; i.e., $\sup_{e \in [c,1-c]} |B_f(e)| + |B'_f(e)| + |B''_{f}(e)| < \infty $. 
        \item \textbf{Loss regularity:} For each fixed $e \in [c,1-c]$, the map $\vartheta \mapsto \ell(V;\vartheta,e)$ is twice continuously differentiable, with 
        \begin{align*}
            \sup_{\vartheta,e} \|\ell(V;\vartheta,e)\|^2_2 < \infty, \;\;\; \sup_{\vartheta,e} \|\nabla_\theta \ell(V;\vartheta,e)\|^2_2 < \infty, \;\;\; \sup_{\theta,e}\|\nabla^2_{\vartheta\vartheta}\ell(V;\vartheta,e)\|^2_2 < \infty.
        \end{align*}
        \item \textbf{Higher-order smoothness:} Let $H(\vartheta; e) \triangleq \nabla^{2}_{\vartheta\vartheta}R(\vartheta; e)$. There exists a neighborhood $\varTheta_0$ of $\vartheta$ containing $\vartheta_0$ and constants $0 < \kappa \leq \kappa_2 < \infty$ such that 
        \begin{align}
            \kappa_1 \mathbf{I} \;\preceq\; H(\vartheta; e) \;\preceq\; \kappa_2 \mathbf{I} \quad\text{for all }\vartheta\in\varTheta_0.
        \end{align}
        \item \textbf{Uniform LLN:} Define the empirical risk w.r.t. $\ell$ with the training fold is $\widehat{R}(\vartheta, \widehat{e})$. Then, we have a uniform law-of-large-number:
        \begin{align}
            \sup_{\vartheta} \big| \widehat R(\vartheta; \widehat e) - R(\vartheta; \widehat e) \big| = O_p(n^{-1/2}).
        \end{align}
        \item \textbf{Bounded parameters:} $h_a, u_a$ are bounded by some constant $M$. 
        \item \textbf{Smoothness of $g^{\ast}$:} The convex conjugate $g^{\ast}$ is continuously differentiable with bounded derivative; i.e., $\sup_{t \in \mathcal{T}}| (g^{\ast})'(t)| < \infty$ where $\mathcal{T}$ is a range where $g^{\ast}(t)$ is well-defined. 
    \end{enumerate}
\end{assumption}

The following theorem establishes the convergence rate for the marginal case:
\begin{theorem}[\textbf{Error Analysis (Marginal Case)}]\label{thm:error-ate}
\normalfont
Assume Assumption~\ref{assumption:regularity-condition-3}. Let $e_0\triangleq\{e_{0,a}:a\in\mathcal A\}$
with $e_{0,a}\equiv\Pr(A=a)$ and let $\widehat e_a\equiv n_a/n$. Let
$\vartheta_0\in\arg\min_{\vartheta\in\Theta} R(\vartheta;e_0)$ and
$\widehat\vartheta\in\arg\min_{\vartheta\in\Theta}\widehat R_n(\vartheta;\widehat e)$.
Let $\overline\theta_\varphi$ and $\widehat\theta_\varphi$ be the population target and the estimator
defined in Def.~\ref{def:bound-average} (marginal case).
Then
\begin{align}
\|\widehat\vartheta-\vartheta_0\|_2^2 = O_p(n^{-1/2}),
\qquad
\|\widehat\theta_\varphi-\overline\theta_\varphi\|_2^2 = O_p(n^{-1/2}).
\end{align}
\end{theorem}

Thm.~\ref{thm:error-ate} shows that in the marginal case, both the dual parameters and the bound estimator achieve a \emph{squared} error rate of $O_p(n^{-1/2})$ (implying a parameter convergence rate of $O_p(n^{-1/4})$) without requiring debiasing. This is because the marginal propensity score $\widehat{e}_a = n_a/n$ converges at rate $O_p(n^{-1/2})$, which is fast enough that first-order bias terms vanish asymptotically. This contrasts with the conditional case (Thm.~\ref{thm:error-analysis}), where debiasing is essential to handle slower convergence rates of nonparametric nuisance estimators.
\section{Experiments}
\label{sec:experiments}

This section empirically validates our framework across both synthetic and real-world data. Our goal is to bound the conditional causal mean $\theta(1,x) \triangleq \mathbb{E}[Y \mid \mathrm{do}(A=1), X=x]$ using our proposed bounds in Def.~\ref{def:cross-fit}. All implementation can be found in the \href{https://github.com/yonghanjung/Information-Theretic-Bounds}{GitHub repository}. 

Across all experiments, we estimate the propensity score via XGBoost \citep{chen2016xgboost} and fit the dual functions $\lambda(a,x)=\exp(h(a,x))$ and $u(\cdot)$ using a neural network trained with two-fold cross-fitting. We consider the $f$-divergences in Cor.~\ref{cor:f-divergence} (KL, Jensen--Shannon, Hellinger, TV, and $\chi^2$), and the order-statistics aggregator (Def.~\ref{def:kth}). In the figures below, the label \texttt{tight\_kth} denotes the aggregated interval constructed by trimming the most extreme upper and lower bounds (i.e., $k_{up}=1, k_{lo}=1$) from the computed $f$-divergences.

\begin{figure}[t]
    \centering
    \begin{minipage}[b]{0.32\textwidth}
        \centering
        \subfloat[]{
            \adjincludegraphics[scale=0.18,trim={{.0\width} {.0\width} {.0\width} {.0\width}},clip]{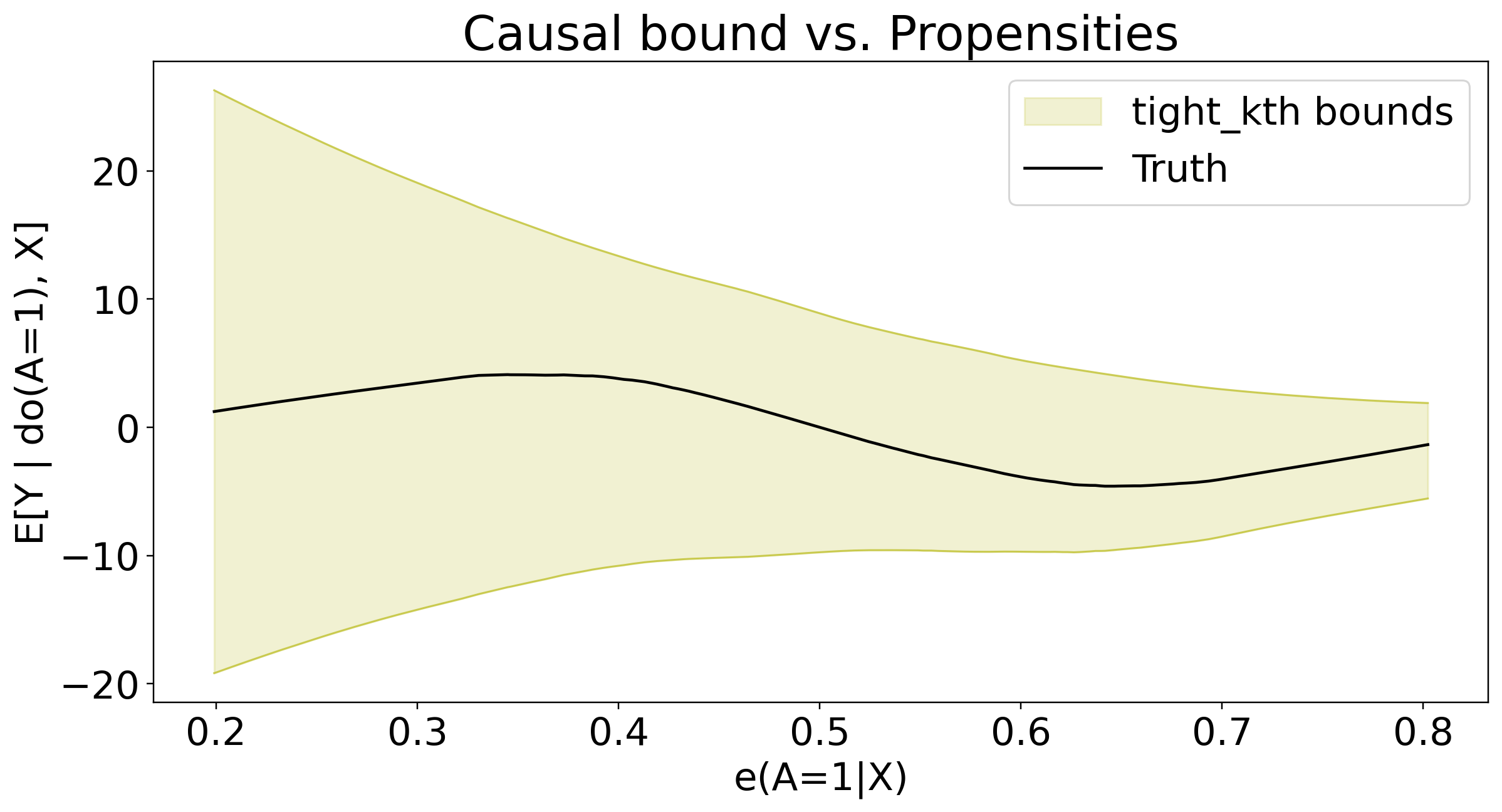}
            \label{fig:synthetic-a}
        }
    \end{minipage}\hfill
    \begin{minipage}[b]{0.32\textwidth}
        \centering
        \subfloat[]{
            \adjincludegraphics[scale=0.27,trim={{.0\width} {.0\width} {.0\width} {.0\width}},clip]{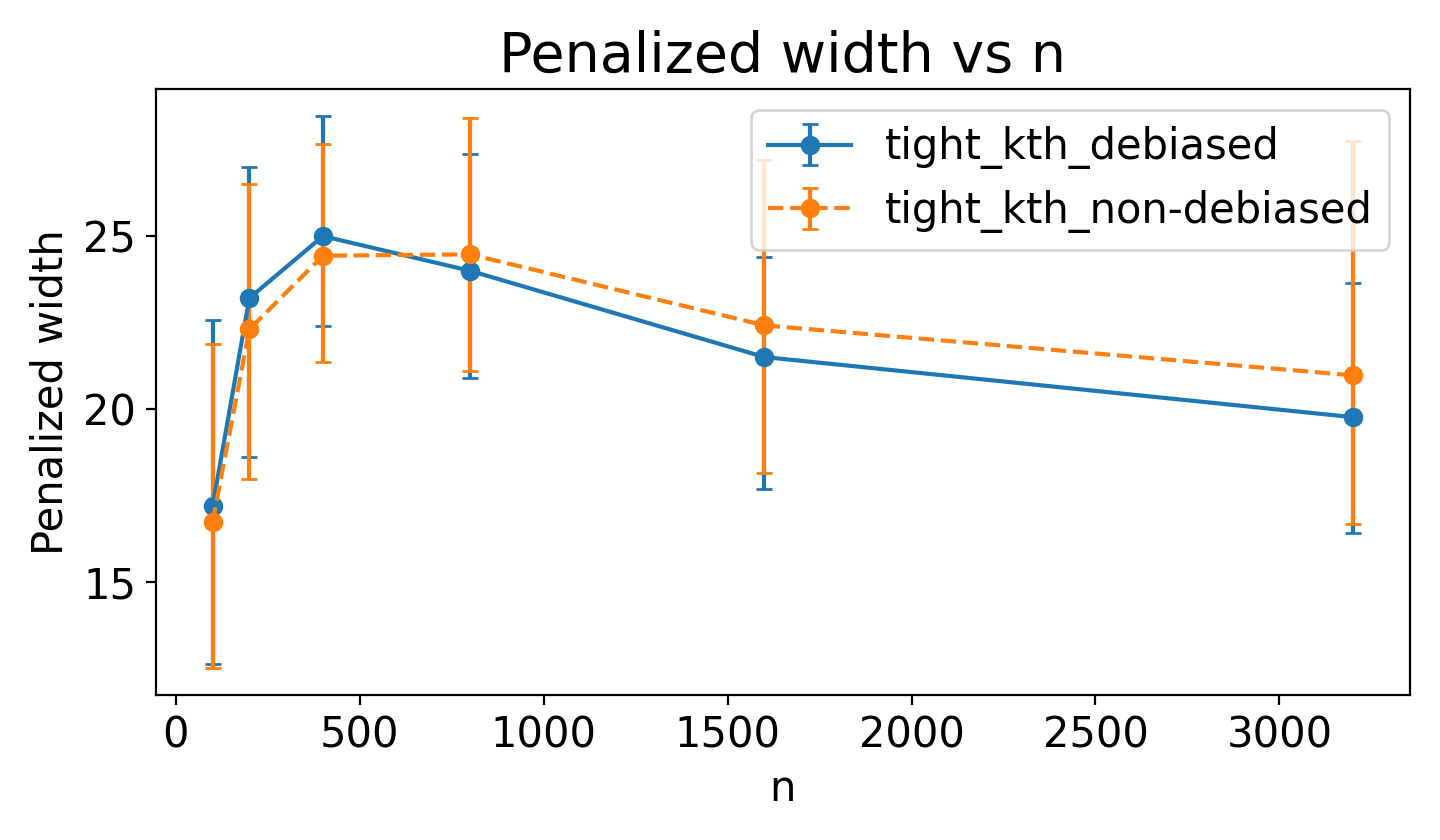}
            \label{fig:penalized-width}
        }
    \end{minipage}\hfill
    \begin{minipage}[b]{0.32\textwidth}
        \centering
        \subfloat[]{
            \adjincludegraphics[scale=0.27,trim={{.0\width} {.0\width} {.0\width} {.0\width}},clip]{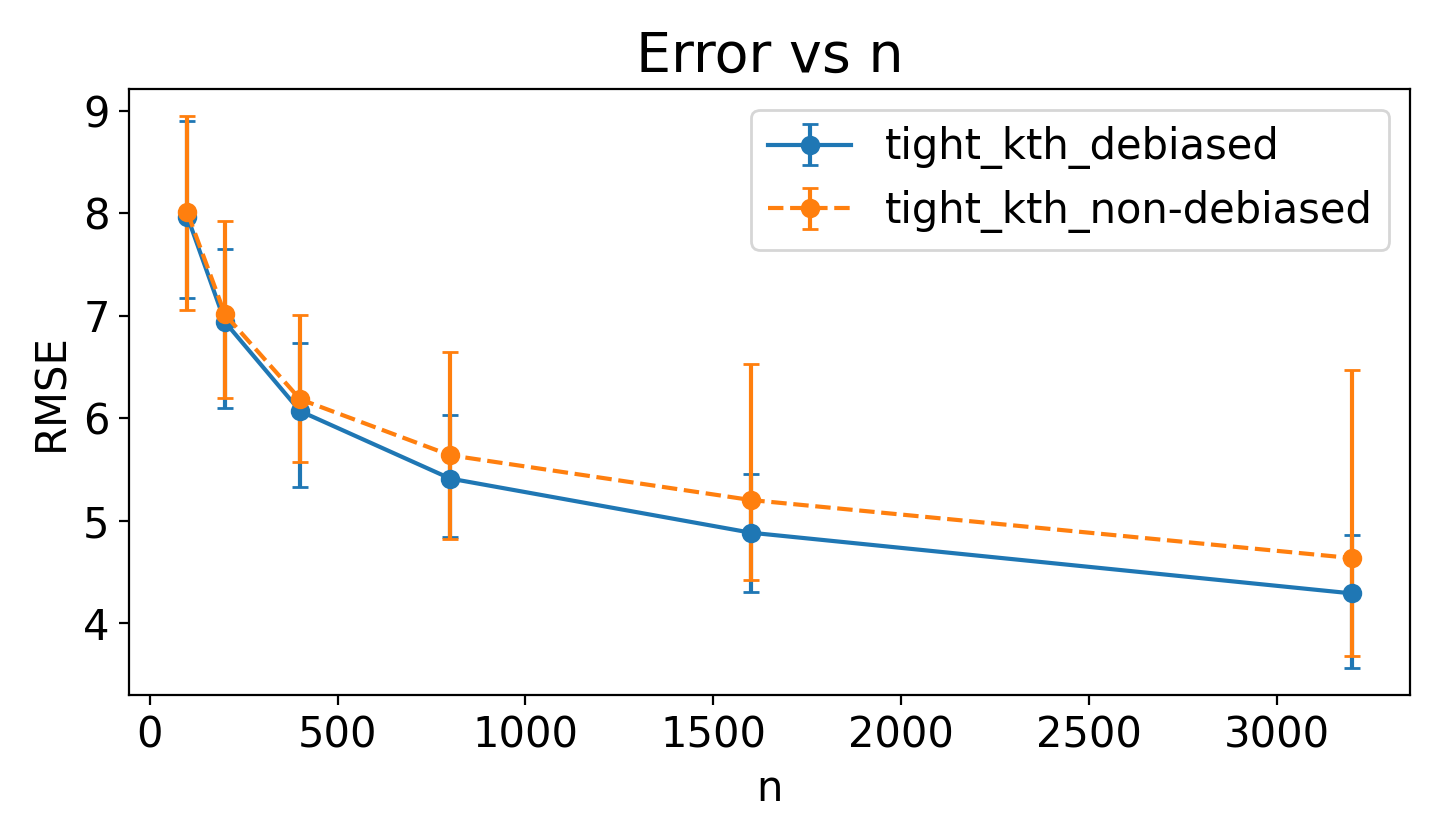}
            \label{fig:rmse}
        }
    \end{minipage}
    \caption{\textbf{(a)} Bounds vs. propensity scores; \textbf{(b)} Penalized width vs. sample size, where the penalized width is $\text{p-width}\triangleq \mathrm{width} \times (1+ a \times \max (0, (1-\alpha)-\mathrm{coverage}))$ with $a=10$ and $\alpha=0.95$; \textbf{(c)} Convergence rate comparison}
    \label{fig:synthetic}
\end{figure}

\paragraph{Synthetic data experiments.} We generate synthetic data from the SCM in Fig.~\ref{fig:dgp} with $X\in\mathbb{R}^5$, binary treatment $A\in\{0,1\}$, and a continuous outcome $Y$ with heavy-tailed noise following a Student's t-distribution with 3 degrees of freedom, which has substantially thicker tails than the standard normal distribution. Fig.~\ref{fig:synthetic-a} demonstrates the validity of our method: the true effect curve for $\theta(1,x)$ lies within the estimated \texttt{tight\_kth} bounds across all propensity score regimes, even under heavy-tailed noise. This plot also shows how interval width shrinks as $e_a(x) \rightarrow 1$, as expected from our theory.

We next examine the debiasing benefit formalized in Thm.~\ref{thm:error-analysis}. Fig.~\ref{fig:penalized-width} compares penalized width, defined as $\text{p-width}\triangleq \mathrm{width} \times (1+ a \times \max (0, (1-\alpha)-\mathrm{coverage}))$ with $a=10$ and $\alpha=0.95$, where $\mathrm{coverage}$ denotes the fraction of evaluation points $\{x_1,\cdots,x_n\}$ satisfying $\theta(1,x_i) \in [\widehat{\theta}_{\mathrm{lo}}(1,x_i),\widehat{\theta}_{\mathrm{up}}(1,x_i)]$, between debiased and non-debiased estimators. As expected, the debiased estimator achieves tighter penalized width as $n$ increases, reflecting improved finite-sample efficiency. Fig.~\ref{fig:rmse} further illustrates robustness to nuisance estimation error: we add convergence noise $\epsilon \sim \mathcal{N}(n^{-1/4}, n^{-1/4})$ to the estimated propensity score and compare convergence rates using an oracle estimator equipped with the true propensity score. The debiased estimator maintains its convergence rate despite slower propensity score estimation, confirming the theoretical guarantee of first-order insensitivity.

\begin{wrapfigure}{r}{.4\textwidth}
    \vspace{-0.3in}
    \centering
    \adjincludegraphics[scale=0.35,trim={{.0\width} {.0\width} {.0\width} {.0\width}},clip]{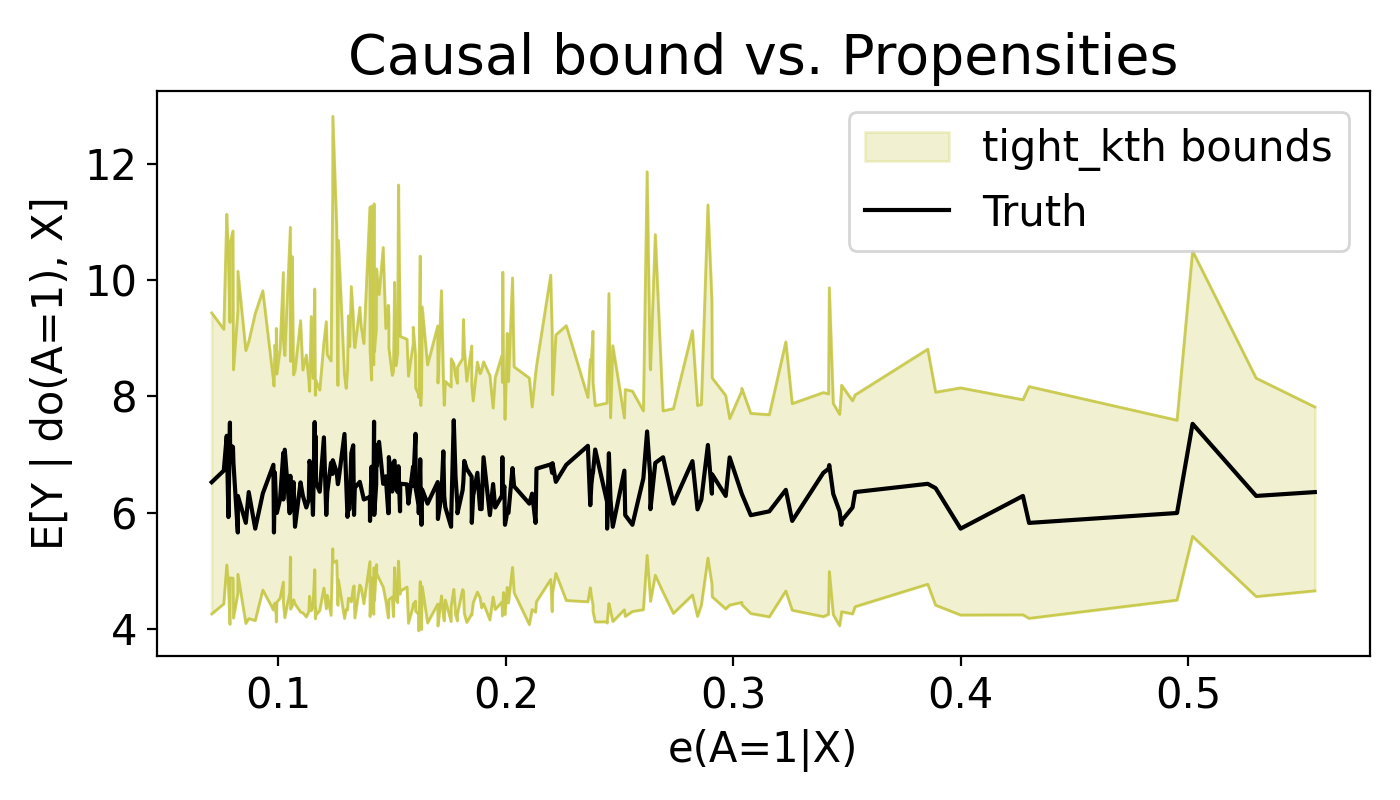}
    \caption{IHDP data analysis}
    \label{fig:ribbon-ihdp}
\end{wrapfigure}

\paragraph{Semi-synthetic IHDP benchmark.} We also validate our method on the well-known IHDP (Infant Health and Development Program) benchmark \citep{hill2011bayesian,louizos2017causal,amlab_cevae_github2020}. This dataset originates from a randomized trial studying the effect of home visits by specialists on future cognitive test scores, with confounders $X \in \mathbb{R}^{25}$ capturing characteristics of the children and their mothers. Following \citet{louizos2017causal}, we de-randomize the treatment assignment to introduce confounding. In our experiment, we observe only five covariates and treat the remaining 20 as hidden confounders. Fig.~\ref{fig:ribbon-ihdp} confirms that our bounds tightly contain the true causal effect $\mathbb{E}[Y \mid \mathrm{do}(A=1),X=x]$ across the full range of estimated propensity scores $\widehat{e}_1(x)$ for $x \in \mathcal{X}$.

\section{Conclusion}
This paper develops an information-theoretic framework for partial identification of causal effects under unmeasured confounding. The key contribution is deriving data-driven bounds on $f$-divergences between observational and interventional distributions using only the propensity score, without requiring auxiliary variables or user-specified sensitivity parameters. These divergence bounds translate into causal effect bounds that simultaneously address four key limitations of existing methods: (1) accommodating unbounded continuous outcomes, (2) avoiding full structural causal model specification, (3) providing heterogeneous effect bounds conditional on covariates, and (4) achieving computational tractability through debiased semiparametric estimation. Our debiased semiparametric estimators achieve $\sqrt{n}$-consistency even when nuisance components converge at slower nonparametric rates, leveraging Neyman-orthogonality to eliminate first-order bias. Experiments on synthetic and semi-synthetic benchmarks confirm valid coverage across propensity score regimes and demonstrate robustness to heavy-tailed outcome distributions. Future work includes extending the framework to continuous treatments and deriving sharper bounds by incorporating additional structural information or auxiliary data.

\endgroup

\newpage
\bibliographystyle{apalike}
\bibliography{reference}

\begin{thebibliography}{}

\bibitem[Ali and Silvey, 1966]{ali1966general}
Ali, S.~M. and Silvey, S.~D. (1966).
\newblock A general class of coefficients of divergence of one distribution
  from another.
\newblock {\em Journal of the Royal Statistical Society: Series B
  (Methodological)}, 28(1):131--142.
\newblock S2PaperId: 77ce3697fc01e0bebba1dfe81cedb712b0b604a0.

\bibitem[{AMLab Amsterdam}, 2020]{amlab_cevae_github2020}
{AMLab Amsterdam} (2020).
\newblock Amlab-amsterdam/cevae: Causal effect inference with deep
  latent-variable models.
\newblock GitHub repository.
\newblock Archived July 17, 2020.

\bibitem[Balazadeh~Meresht et~al., 2022]{balazadeh2022partial}
Balazadeh~Meresht, V., Syrgkanis, V., and Krishnan, R.~G. (2022).
\newblock Partial identification of treatment effects with implicit generative
  models.
\newblock {\em Advances in Neural Information Processing Systems},
  35:22816--22829.
\newblock S2PaperId: fe141b34c03a3e5f830fee948499f55e59ea8639.

\bibitem[Balke and Pearl, 1994]{balke1994counterfactual}
Balke, A. and Pearl, J. (1994).
\newblock Counterfactual probabilities: Computational methods, bounds and
  applications.
\newblock In {\em Uncertainty in artificial intelligence}, pages 46--54.
  Elsevier.
\newblock S2PaperId: 4f48b2ba8c3b8e8f0aa3d61e3f30c5c66997c7ab.

\bibitem[Balke and Pearl, 1997]{balke1997bounds}
Balke, A. and Pearl, J. (1997).
\newblock Bounds on treatment effects from studies with imperfect compliance.
\newblock {\em Journal of the American statistical Association},
  92(439):1171--1176.
\newblock S2PaperId: 7abc07c41b2bade96bc52f10c892f206fbabbe63.

\bibitem[Bretagnolle and Huber, 1979]{bretagnolle1979estimation}
Bretagnolle, J. and Huber, C. (1979).
\newblock Estimation des densit{\'e}s: risque minimax.
\newblock {\em Zeitschrift f{\"u}r Wahrscheinlichkeitstheorie und verwandte
  Gebiete}, 47(2):119--137.

\bibitem[Chen and Guestrin, 2016]{chen2016xgboost}
Chen, T. and Guestrin, C. (2016).
\newblock Xgboost: A scalable tree boosting system.
\newblock In {\em Proceedings of the 22nd ACM SIGKDD International Conference
  on Knowledge Discovery and Data Mining}, pages 785--794.
\newblock S2PaperId: 26bc9195c6343e4d7f434dd65b4ad67efe2be27a.

\bibitem[Chernozhukov et~al., 2018]{chernozhukov2018double}
Chernozhukov, V., Chetverikov, D., Demirer, M., Duflo, E., Hansen, C., Newey,
  W., and Robins, J. (2018).
\newblock Double/debiased machine learning for treatment and structural
  parameters: Double/debiased machine learning.
\newblock {\em The Econometrics Journal}, 21(1).
\newblock S2PaperId: f75b70c9d7078724b592ec3e21de705e7b6ff73f.

\bibitem[Csisz{\'a}r, 1967]{csiszar1967fdivergenceDPI}
Csisz{\'a}r, I. (1967).
\newblock Information-type measures of difference of probability distributions
  and indirect observations.
\newblock {\em Studia Scientiarum Mathematicarum Hungarica}, 2:299--318.

\bibitem[Dorn and Guo, 2023]{dorn2023sharp}
Dorn, J. and Guo, K. (2023).
\newblock Sharp sensitivity analysis for inverse propensity weighting via
  quantile balancing.
\newblock {\em Journal of the American Statistical Association},
  118(544):2645--2657.
\newblock S2PaperId: eb6fecfd3e1b4839868a18111434e2616269b834.

\bibitem[Ghassami et~al., 2023]{ghassami2023partial}
Ghassami, A., Shpitser, I., and Tchetgen, E.~T. (2023).
\newblock Partial identification of causal effects using proxy variables.
\newblock {\em arXiv preprint arXiv:2304.04374}.

\bibitem[Gretton et~al., 2012]{gretton2012kernel}
Gretton, A., Borgwardt, K.~M., Rasch, M.~J., Sch{\"o}lkopf, B., and Smola, A.
  (2012).
\newblock A kernel two-sample test.
\newblock {\em The journal of machine learning research}, 13(1):723--773.
\newblock S2PaperId: 225f78ae8a44723c136646044fd5c5d7f1d3d15a.

\bibitem[Hill, 2011]{hill2011bayesian}
Hill, J.~L. (2011).
\newblock Bayesian nonparametric modeling for causal inference.
\newblock {\em Journal of Computational and Graphical Statistics},
  20(1):217--240.

\bibitem[Hu et~al., 2021]{hu2021generative}
Hu, Y., Wu, Y., Zhang, L., and Wu, X. (2021).
\newblock A generative adversarial framework for bounding confounded causal
  effects.
\newblock In {\em Proceedings of the AAAI Conference on Artificial
  Intelligence}, volume~35, pages 12104--12112.
\newblock S2PaperId: f8a27911dff2050b9f0a608bad555df8a0ce34d3.

\bibitem[Jiang and Kocaoglu, 2024]{pmlr-v235-jiang24b}
Jiang, Z. and Kocaoglu, M. (2024).
\newblock Conditional common entropy for instrumental variable testing and
  partial identification.
\newblock In Salakhutdinov, R., Kolter, Z., Heller, K., Weller, A., Oliver, N.,
  Scarlett, J., and Berkenkamp, F., editors, {\em Proceedings of the 41st
  International Conference on Machine Learning}, volume 235 of {\em Proceedings
  of Machine Learning Research}, pages 21824--21843. PMLR.

\bibitem[Jiang et~al., 2023]{pmlr-v202-jiang23h}
Jiang, Z., Wei, L., and Kocaoglu, M. (2023).
\newblock Approximate causal effect identification under weak confounding.
\newblock In Krause, A., Brunskill, E., Cho, K., Engelhardt, B., Sabato, S.,
  and Scarlett, J., editors, {\em Proceedings of the 40th International
  Conference on Machine Learning}, volume 202 of {\em Proceedings of Machine
  Learning Research}, pages 15125--15143. PMLR.

\bibitem[Jin et~al., 2022]{jin2022sensitivity}
Jin, Y., Ren, Z., and Zhou, Z. (2022).
\newblock Sensitivity analysis under the $ f $-sensitivity models: a
  distributional robustness perspective.
\newblock {\em arXiv preprint arXiv:2203.04373}.
\newblock S2PaperId: 60860ee182060844ab7d7da661e787352e956ec4.

\bibitem[Kitagawa, 2021]{kitagawa2021IV}
Kitagawa, T. (2021).
\newblock The identification region of the potential outcome distributions
  under instrument independence.
\newblock {\em Journal of Econometrics}, 225(2):231--253.
\newblock Themed Issue: Treatment Effect 1.

\bibitem[Lee, 2009]{lee2009training}
Lee, D.~S. (2009).
\newblock Training, wages, and sample selection: Estimating sharp bounds on
  treatment effects.
\newblock {\em The Review of Economic Studies}, pages 1071--1102.

\bibitem[Levis et~al., 2025]{levis2025covariate}
Levis, A.~W., Bonvini, M., Zeng, Z., Keele, L., and Kennedy, E.~H. (2025).
\newblock Covariate-assisted bounds on causal effects with instrumental
  variables.
\newblock {\em Journal of the Royal Statistical Society Series B: Statistical
  Methodology}.
\newblock S2PaperId: 1cde98151e800f593543418ce581590d6586aa41.

\bibitem[Louizos et~al., 2017]{louizos2017causal}
Louizos, C., Shalit, U., Mooij, J.~M., Sontag, D., Zemel, R., and Welling, M.
  (2017).
\newblock Causal effect inference with deep latent-variable models.
\newblock {\em Advances in neural information processing systems}, 30.
\newblock S2PaperId: 20d117f0cccf4aaaeadfdeb58d7af0fc201f7a9a.

\bibitem[Manski, 1990]{manski1990nonparametric}
Manski, C.~F. (1990).
\newblock Nonparametric bounds on treatment effects.
\newblock {\em The American Economic Review}, 80(2):319--323.
\newblock S2PaperId: af849d98b521efa161aa81410e148c921316dd05.

\bibitem[M{\"u}ller, 1997]{muller1997integral}
M{\"u}ller, A. (1997).
\newblock Integral probability metrics and their generating classes of
  functions.
\newblock {\em Advances in applied probability}, 29(2):429--443.

\bibitem[Oprescu et~al., 2023]{oprescu2023b}
Oprescu, M., Dorn, J., Ghoummaid, M., Jesson, A., Kallus, N., and Shalit, U.
  (2023).
\newblock B-learner: Quasi-oracle bounds on heterogeneous causal effects under
  hidden confounding.
\newblock In {\em International Conference on Machine Learning}, pages
  26599--26618. PMLR.
\newblock S2PaperId: 4ce536b72ce6cdcb5b215cb010b22d134bd60604.

\bibitem[Padh et~al., 2023]{padh2023stochastic}
Padh, K., Zeitler, J., Watson, D., Kusner, M., Silva, R., and Kilbertus, N.
  (2023).
\newblock Stochastic causal programming for bounding treatment effects.
\newblock In {\em Conference on Causal Learning and Reasoning}, pages 142--176.
  PMLR.
\newblock S2PaperId: 8303853090707a8f9589376db5623a5c0f0d5308.

\bibitem[Pearl, 2000]{pearl:2k}
Pearl, J. (2000).
\newblock {\em Causality: Models, Reasoning, and Inference}.
\newblock Cambridge University Press, New York.
\newblock 2nd edition, 2009.

\bibitem[Rosenbaum, 1987]{rosenbaum1987sensitivity}
Rosenbaum, P.~R. (1987).
\newblock Sensitivity analysis for certain permutation inferences in matched
  observational studies.
\newblock {\em Biometrika}, 74(1):13--26.

\bibitem[Sachs et~al., 2023]{sachs2023general}
Sachs, M.~C., Jonzon, G., Sj{\"o}lander, A., and Gabriel, E.~E. (2023).
\newblock A general method for deriving tight symbolic bounds on causal
  effects.
\newblock {\em Journal of Computational and Graphical Statistics},
  32(2):567--576.
\newblock S2PaperId: bf30ac1e2f5d6596e6d76ba0d84e88c13d8907f4.

\bibitem[Semenova, 2025]{semenova2025generalized}
Semenova, V. (2025).
\newblock Generalized lee bounds.
\newblock {\em Journal of Econometrics}, 251:106055.
\newblock S2PaperId: 398e59a9c049767272a9cd6789ea75ef7b296b65.

\bibitem[Shridharan and Iyengar, 2023]{shridharan2023scalable}
Shridharan, M. and Iyengar, G. (2023).
\newblock Scalable computation of causal bounds.
\newblock {\em Journal of Machine Learning Research}, 24(237):1--35.
\newblock S2PaperId: d3467ee168161b4aec913b1ce2b0c40774b766f9.

\bibitem[Swanson et~al., 2018]{swanson2018partial}
Swanson, S.~A., Hern{\'a}n, M.~A., Miller, M., Robins, J.~M., and Richardson,
  T.~S. (2018).
\newblock Partial identification of the average treatment effect using
  instrumental variables: review of methods for binary instruments, treatments,
  and outcomes.
\newblock {\em Journal of the American Statistical Association},
  113(522):933--947.
\newblock S2PaperId: 9cebaa4ad91b84ae3e1edc1350d9719b594e42fd.

\bibitem[Tan et~al., 2024]{tan2024consistency}
Tan, J., Blanchet, J., and Syrgkanis, V. (2024).
\newblock Consistency of neural causal partial identification.
\newblock {\em Advances in Neural Information Processing Systems}, 37.
\newblock arXiv:2405.15673.

\bibitem[Tan, 2006]{tan2006distributional}
Tan, Z. (2006).
\newblock A distributional approach for causal inference using propensity
  scores.
\newblock {\em Journal of the American Statistical Association},
  101(476):1619--1637.
\newblock S2PaperId: bd57ce8e871b42827c4df1ed99def8c578ac7b9e.

\bibitem[Tian and Pearl, 2000]{tian2000probabilities}
Tian, J. and Pearl, J. (2000).
\newblock Probabilities of causation: Bounds and identification.
\newblock {\em Annals of Mathematics and Artificial Intelligence},
  28(1):287--313.

\bibitem[Xia et~al., 2022]{xia2022neural}
Xia, K.~M., Pan, Y., and Bareinboim, E. (2022).
\newblock Neural causal models for counterfactual identification and
  estimation.
\newblock In {\em The Eleventh International Conference on Learning
  Representations}.

\bibitem[Yadlowsky et~al., 2022]{yadlowsky2022bounds}
Yadlowsky, S., Namkoong, H., Basu, S., Duchi, J., and Tian, L. (2022).
\newblock Bounds on the conditional and average treatment effect with
  unobserved confounding factors.
\newblock {\em Annals of statistics}, 50(5):2587.

\bibitem[Zhang and Bareinboim, 2021]{zhang2021bounding}
Zhang, J. and Bareinboim, E. (2021).
\newblock Bounding causal effects on continuous outcome.
\newblock In {\em Proceedings of the AAAI Conference on Artificial
  Intelligence}, volume~35, pages 12207--12215.

\end{thebibliography}


\clearpage
\appendix 




\onecolumn
\vtop{
    \centering
    \vspace{0cm}
    \hspace{0cm}
    \begin{tikzpicture}
        \centering 
        \node (box){
          \begin{minipage}[t!]{0.9\textwidth}
            \centering
            {\Large Supplement of \Paste{title} }
          \end{minipage}
        };
    \end{tikzpicture}
    \vspace{0.5cm}
}
\begingroup
\allowdisplaybreaks
\sloppy
\section{Simulation Details}
\label{app:simulations}

This section provides the technical specifications for the synthetic and semi-synthetic experiments presented in Section~\ref{sec:experiments}.

\subsection{Synthetic Data Generating Process}
\label{app:sim:synthetic}

To evaluate the performance of our proposed information-theoretic bounds in a controlled yet challenging environment, we design a synthetic data generating process (DGP) based on a probit-style structural causal model (SCM). This setup allows us to precisely manipulate the degree of unmeasured confounding and the complexity of the treatment effect, providing a rigorous testbed for our debiased estimation framework.

\paragraph{Feature and Confounder Generation}
We consider a feature space $X \in \mathbb{R}^d$, where the first $d-1$ dimensions represent standard Gaussian noise, $X_j \sim \mathcal{N}(0,1)$ for $j=1,\dots,d-1$. The coordinate $X_0$ is designated as the primary observed covariate influencing both treatment and outcome, with its variance scaled as $X_0 \sim \mathcal{N}\!\left(0,\left(\sqrt{1+\beta^2}/\alpha\right)^2\right)$ to maintain numerical stability across different confounding regimes. A latent confounder $U \sim \mathcal{N}(0,1)$ is introduced to represent unmeasured factors that simultaneously affect the treatment assignment and the outcome, thereby creating the hidden confounding scenario our method aims to address.

\paragraph{Treatment Assignment and Propensity Score}
The binary treatment assignment $A \in \{0, 1\}$ is generated via a probit mechanism. We first define a latent score $s(X_0, U)$ that linearly combines the observed feature and the hidden confounder:
\begin{equation}
    s(X_0, U) = \alpha X_0 + \beta U.
\end{equation}
The parameters $\alpha$ and $\beta$ play crucial roles in our simulation: $\alpha$ controls the strength of the observed signal in the selection process, while $\beta$ determines the magnitude of unmeasured confounding. In our default setting, we fix $(\alpha, \beta) = (2, 1)$. The treatment is then sampled as $A \sim \text{Bernoulli}(e(X_0, U))$, where the propensity score $e(X_0, U)$ is defined as:
\begin{equation}
    e(X_0, U) = c + (1-2c)\,\Phi\!\left(s(X_0,U)\right).
\end{equation}
Here, $\Phi(\cdot)$ denotes the standard normal cumulative distribution function (CDF). We set $c=0.05$ to enforce the overlap condition, ensuring that the propensity values are bounded within $[0.05, 0.95]$ and preventing the total absence of either treatment or control units in localized regions of the feature space.

\paragraph{Outcome Generation and Treatment Effect}
The outcome $Y$ is modeled as a linear combination of the treatment effect, the hidden confounder, and additive noise:
\begin{equation}
    Y = \tau(X_0)A + \gamma U + \epsilon,
\end{equation}
where $\gamma = 1$ scales the influence of $U$ on the outcome. To test the robustness of our bounds against non-linear signals, we define the true conditional average treatment effect (CATE) $\tau(X_0)$ as a sinusoidal function of the marginal propensity score $\bar e(X_0)$:
\begin{align}
    \bar e(X_0) &= c + (1-2c)\,\Phi\!\left(\frac{\alpha X_0}{\sqrt{1+\beta^2}}\right), \\
    \tau(X_0) &= 5\sin\!\left(2\pi\frac{\bar e(X_0)-c}{1-2c}\right).
\end{align}
This formulation ensures that the treatment effect varies nonlinearly across the population. Finally, we vary the noise distribution $\epsilon$ to simulate different data conditions: we use heavy-tailed noise $\epsilon \sim t_3(0,1)$ for the visualization in Fig.~\ref{fig:synthetic-a}, and standard Gaussian noise $\epsilon \sim \mathcal{N}(0,1)$ for the sample-size and error analysis experiments in Fig.~\ref{fig:penalized-width} and Fig.~\ref{fig:rmse}.

\subsection{Neural Network Architecture and Training}
\label{app:sim:nn}

For estimating the dual functions and the nuisance components (outcome models and propensity scores), we use Multi-Layer Perceptrons (MLPs) and XGBoost.

\paragraph{Architecture for Dual Functions} 
The dual functions $h$ are parameterized by an MLP with two hidden layers of 64 units each, using ReLU activations. We apply a clipping operation to the output of the dual network such that $h(X) \in [-20, 20]$ to ensure numerical stability during optimization. For the synthetic ribbon experiment (Fig.~\ref{fig:synthetic-a}), we use a dropout rate of $0.1$; for the synthetic sample-size/error experiments (Fig.~\ref{fig:penalized-width} and Fig.~\ref{fig:rmse}) and the IHDP experiments, we use no dropout.

\paragraph{Optimization and Hyperparameters}
The dual networks are trained using the Adam optimizer with a learning rate of $5 \times 10^{-4}$ and weight decay of $1 \times 10^{-4}$. We employ 2-fold cross-fitting to avoid overfitting and ensure the validity of the debiased estimator. For each fold, we train the dual network for up to 1000 epochs in the synthetic ribbon and IHDP experiments, and 2000 epochs in the synthetic varying-sample-size experiments. Early stopping with a patience of 10 epochs (monitored on a $20\%$ validation split of the training fold) is used to prevent overfitting.

\paragraph{Nuisance Models}
Propensity scores and outcome means are estimated using XGBoost with the following hyperparameters:
\begin{itemize}
    \item \textbf{Number of estimators}: 300 for propensity, 400 for outcome.
    \item \textbf{Maximum depth}: 10.
    \item \textbf{Learning rate}: 0.005.
    \item \textbf{Subsample / Colsample}: 0.8.
\end{itemize}

\subsection{IHDP Benchmark Details}
\label{app:sim:ihdp}

The IHDP benchmark is a semi-synthetic dataset based on a real-world randomized trial from the Infant Health and Development Program. We use the version where selection bias is introduced by removing a non-random subset of the treated group. 

In our experiments, we treat 5 of the 25 covariates as observed and the remaining 20 as hidden confounders to simulate a scenario with unmeasured confounding. The evaluation is performed on a fixed set of units to compare the estimated bounds against the ground truth interventional effects provided by the benchmark. Training is conducted for 1000 epochs for the IHDP-specific experiments.

\section{Proofs}\label{app:proof}
\subsection*{Proof of Thm.~\ref{thm:f-divergence}}

We first declare some useful results: 
\begin{lemma}[\textbf{f-divergence with Conditional Measure}]\label{lemma:f-divergence-conditional}
    \normalfont
    Let $P$ on $(\Omega, \mathcal{F})$ be an arbitrary probability measure. Let $E \in \mathcal{F}$ be a fixed event such that $P(E) = p \in (0,1)$. Let $P_E(\cdot) \triangleq P(\cdot \mid E)$. Then, 
    \begin{align}
        D_{f}(P_E \| P) = pf(\tfrac{1}{p}) + (1-p)f(0)
    \end{align}
\end{lemma}
\begin{proof}
Define the conditional-on-an-event measure $P_E$ by
\begin{align}
P_E(B) \coloneqq P(B\mid E)=\frac{P(B\cap E)}{P(E)},\qquad \forall B\in\mathcal{F},
\end{align}
where $P(E)=p\in(0,1)$. Then $P_E\ll P$ since $P(B)=0\Rightarrow P(B\cap E)=0\Rightarrow P_E(B)=0$.
Hence, by the Radon--Nikodým theorem, there exists a measurable function
$g=\frac{dP_E}{dP}$ (unique $P$-a.e.) such that
\begin{align}
P_E(B)=\int_B g(\omega)\,P(d\omega),\qquad \forall B\in\mathcal{F}.
\end{align}
A valid version is $g(\omega)=\frac{1}{p}\,\Ind_E(\omega)$ since, for any $B\in\mathcal{F}$,
\begin{align}
\int_B \frac{1}{p}\Ind_E(\omega)\,P(d\omega)
=\frac{1}{p}P(B\cap E)=\frac{P(B\cap E)}{P(E)}=P_E(B).
\end{align}
Therefore,
\begin{align}
D_f(P_E\|P)
&=\int_\Omega f\!\left(\frac{dP_E}{dP}(\omega)\right)P(d\omega) \\ 
&=\int_E f\!\left(\frac{1}{p}\right)P(d\omega)+\int_{E^c} f(0)\,P(d\omega) \\
&=pf\!\left(\frac{1}{p}\right)+(1-p)f(0).
\end{align}
\end{proof}

\begin{lemma}[\textbf{Data Processing Inequality} \citep{csiszar1967fdivergenceDPI}]\label{lemma:DPI}
    \normalfont
    Let $P_X$ and $Q_X$ denote probability measures on $(\mathcal{X}, \mathcal{F}_X)$. Let $P_{Y \mid X}$ be a Markov kernel from $(\mathcal{X}, \mathcal{F}_X)$ to $(\mathcal{Y}, \mathcal{F}_Y)$. Let $P_Y, Q_Y$ be the transformation of $P_X,Q_X$, respectively, when pushed through $P_{Y\mid X}$; i.e., $P_Y(B) = \int_{\mathcal{X}}P_{Y \mid X}(B \mid x)dP_X(x)$, and $Q_Y$ is defined similarly. Then, for any $f$-divergence, we have  
    \begin{align}
        D_f(P_Y \| Q_Y) \leq D_f(P_X \| Q_X).
    \end{align}
\end{lemma}

For any fixed $X=x$, define the event $E\coloneqq\{A=a\}$ under the measure $P_{U,A\mid X=x}$, so that
$P(E\mid x)=P(A=a\mid X=x)=e_a(x)$.
Let
\begin{align}
P_E(\cdot\mid x)\coloneqq P_{U,A\mid X=x}(\,\cdot\,\mid E)
= P_{U,A\mid X=x,A=a}.
\end{align}
By Lemma~\ref{lemma:f-divergence-conditional},
\begin{align}
D_f\!\left(P_{U,A\mid X=x,A=a}\,\big\|\,P_{U,A\mid X=x}\right)
= e_a(x)f\!\left(\frac{1}{e_a(x)}\right)+(1-e_a(x))f(0)
\equiv B_f(e_a(x)).
\end{align}

Define the (Markov) transition kernel $K_{a,x}$ from $(\mathcal{U}\times\mathcal{A},\mathcal{F}_{U,A})$
to $(\mathcal{Y},\mathcal{F}_Y)$ by, for any $B\in\mathcal{F}_Y$,
\begin{align}
K_{a,x}(B\mid u,a') \coloneqq P(Y\in B\mid U=u, A=a, X=x),
\end{align}
(note $K_{a,x}$ is constant in $a'$).

Pushing $P_{U,A\mid X=x}$ through $K_{a,x}$ yields
\begin{align}
\int K_{a,x}(B\mid u,a')\,P_{U,A\mid X=x}(du\,da')
= \int P(Y\in B\mid u,a,x)\,P_{U,A\mid X=x}(du\,da')
= P(Y\in B\mid \mathrm{do}(A=a),X=x).
\end{align}
Similarly, pushing $P_{U,A\mid X=x,A=a}$ through $K_{a,x}$ yields
\begin{align}
\int K_{a,x}(B\mid u,a')\,P_{U,A\mid X=x,A=a}(du\,da')
= \int P(Y\in B\mid u,a,x)\,P_{U\mid X=x,A=a}(du)
= P(Y\in B\mid A=a,X=x).
\end{align}
By the data processing inequality (Lemma~\ref{lemma:DPI}),
\begin{align}
D_f\!\left(P_{Y\mid A=a,X=x}\,\big\|\,P_{Y\mid \mathrm{do}(A=a),X=x}\right)
\le
D_f\!\left(P_{U,A\mid X=x,A=a}\,\big\|\,P_{U,A\mid X=x}\right)
= B_f(e_a(x)).
\end{align}
\hfill $\blacksquare$.

\subsection*{Proof of Cor.~\ref{cor:f-divergence}}

\paragraph{KL.} With $f(t) = t\log t$ with $f(0) = 0$, we have 
\begin{align}
    B(e_a(x), f) = -e_a(x)\frac{1}{e_a(x)}\log e_a(x) = -\log e_a(x). 
\end{align}
Therefore, 
\begin{align}
    D_{\mathrm{KL}}(P(Y \mid a,x) \| Q(Y \mid a,x)) \leq -\log e_a(x). 
\end{align}

\paragraph{Hellinger.} With $f(t) = \tfrac{1}{2}(\sqrt{t}-1)^2 $ with $f(0) = 1/2$, we have 
\begin{align}
    B(e_a(x), f) &= e_a(x)f\big(\tfrac{1}{e_a(x)}\big) + (1-e_a(x))f(0)  \\ 
                &= \frac{1}{2}e_a(x) \Big(\sqrt{\tfrac{1}{e_a(x)}} - 1\Big)^2 + \frac{1}{2}(1-e_a(x)) \\ 
                &= 1 - \sqrt{e_a(x)}.
\end{align}

To tighten, we use the following lemma: 
\begin{lemma}[\textbf{Hellinger divergence vs. KL divergence}]\label{lemma:hellinger-KL}
    \normalfont
    For any $P,Q$ such that $P \ll Q$, 
    \begin{align}
        D_{\mathrm{H}}(P \| Q) \leq \frac{1}{2}D_{\mathrm{KL}}(P \| Q). 
    \end{align}
\end{lemma}
\begin{proof}
    We start with 
    \begin{align}
        D_{\mathrm{H}}(P \| Q)  \triangleq \frac{1}{2}\int (\sqrt{p}(x) - \sqrt{q}(x))^2dx = 1-\int \sqrt{p(x)q(x)}dx.
    \end{align}
    Define $\mathtt{BC}(P,Q) \triangleq \int \sqrt{p(x)q(x)}dx$. Then, $D_{\mathrm{H}}(P \| Q)  = 1-\mathtt{BC}(P,Q)$. Define $D_{\mathrm{B}}(P \| Q) \triangleq -\log \mathtt{BC}(P,Q)$, which is known as Bhattacharyya distance. 

    Define $r(X) \triangleq \frac{q(x)}{p(x)}$. Then,
    \begin{align}
        \mathtt{BC}(P,Q) &= \int \sqrt{p(x) q(x)} dx =  \int \sqrt{\tfrac{q(x)}{p(x)}} p(x)dx = \mathbb{E}_{P}\Big[\sqrt{r(X)}\Big]. 
    \end{align}
    By Jensen's inequality, we have 
    \begin{align}
        \log \mathtt{BC}(P,Q)  = \log \mathbb{E}_{P}\Big[\sqrt{r(X)}\Big]  \geq  \mathbb{E}_{P}\Big[\log \sqrt{r(X)}\Big] = \frac{1}{2}\mathbb{E}_{P}[\log r(X)].
    \end{align}
    Also, 
    \begin{align}
        \mathbb{E}_{P}[\log r(X)] = \int p(x) \log \frac{q(x)}{p(x)}dx = -D_{\mathrm{KL}}(P,Q).
    \end{align}
    Combining, 
    \begin{align}
        -\frac{1}{2}D_{\mathrm{KL}}(P \| Q) \leq \log \mathtt{BC}(P,Q) \;\;\Leftrightarrow\;\; 1-\exp\Big(-\frac{1}{2}D_{\mathrm{KL}}(P \| Q)\Big) \geq 1-\mathtt{BC}(P,Q).
    \end{align}
    Finally, 
    \begin{align}
        D_{\mathrm{H}}(P \| Q)  = 1-\mathtt{BC}(P,Q) \leq 1-\exp\Big(-\frac{1}{2}D_{\mathrm{KL}}(P \| Q)\Big) \leq \frac{1}{2}D_{\mathrm{KL}}(P \| Q),
    \end{align}
    where the last inequality holds since $1-e^{-u} \leq u$ for any $u \geq 0$.
\end{proof}
As a result, we can derive 
\begin{align}
    D_{\mathrm{H}}(P(Y \mid a,x) \| Q(Y \mid a,x)) \leq -\frac{1}{2}\log e_a(x). 
\end{align}
Finally, for $e_a(x) \in (0,1)$, the following holds: 
\begin{align}
    1-\sqrt{e_a(x)} \leq -\frac{1}{2}\log e_a(x). 
\end{align}

\paragraph{$\chi^2$-divergence.}
Set $f(t) \triangleq \tfrac{1}{2}(t-1)^2$. Then, $B_f(e_a)  = \tfrac{1-e_a(x)}{2e_a(x)}$.

\paragraph{Total variation.}
First, $B_{f_\mathrm{TV}}(e) = 1-e$.

Second, by Pinsker's inequality and the above inequality, 
\begin{align}
    D_{\mathrm{TV}}(P \| Q) \leq \sqrt{\frac{1}{2}D_{\mathrm{KL}}(P \| Q)} \leq \sqrt{-\frac{1}{2}\log e_a(x)}. 
\end{align}
By Bretagnolle–Huber bound \citep{bretagnolle1979estimation} and the above inequality, 
\begin{align}
    D_{\mathrm{TV}}(P \| Q) \leq \sqrt{ 1-\exp(-D_{\mathrm{KL}}(P \| Q)) } \leq \sqrt{1-e_a(x)}.
\end{align}

Finally, $\min\left( 1-e_a(x), \sqrt{1-e_a(x)}, \sqrt{-\frac{1}{2}\log e_a(x)}\right) = 1-e_a(x)$ for all $e_a(x) \in (0,1)$.  

\paragraph{Jensen-Shannon.}
With $f_{\mathrm{JS}}(t) \triangleq \frac{1}{2}\big(t\log t - (t+1)\log(\frac{t+1}{2})\big)$ and $f_{\mathrm{JS}}(0) = \frac{1}{2}\log 2$, we have:
\begin{align}
    B_{f_{\mathrm{JS}}}(e_a(x)) &= e_a(x)f_{\mathrm{JS}}\Big(\frac{1}{e_a(x)}\Big) + (1-e_a(x))f_{\mathrm{JS}}(0) \\
    &= \frac{e_a(x)}{2} \left[ \frac{1}{e_a(x)}\log\Big(\frac{1}{e_a(x)}\Big) - \Big(\frac{1}{e_a(x)}+1\Big)\log\left(\frac{1+e_a(x)}{2e_a(x)}\right) \right] \nonumber \\ 
    &\quad + \frac{1-e_a(x)}{2}\log 2 \\
    &= \frac{1}{2} \left[ -\log e_a(x) - (1+e_a(x))\log\left(\frac{1+e_a(x)}{2e_a(x)}\right) + (1-e_a(x))\log 2 \right] \\
    &= \frac{1}{2} \left[ -\log e_a(x) - (1+e_a(x))\big[\log(1+e_a(x)) - \log e_a(x) - \log 2\big] + \log 2 - e_a(x)\log 2 \right] \\
    &= \frac{1}{2} \Big[ -\log e_a(x) - (1+e_a(x))\log(1+e_a(x)) + \log e_a(x) \nonumber \\ 
    &\quad \qquad + e_a(x)\log e_a(x) + 2\log 2 + e_a(x)\log 2 - e_a(x)\log 2 \Big] \\
    &= \frac{1}{2} \left[ e_a(x)\log e_a(x) - (1+e_a(x))\log(1+e_a(x)) + 2\log 2 \right] \\
    &= \frac{1}{2} \log \left( \frac{4 e_a(x)^{e_a(x)}}{(1+e_a(x))^{1+e_a(x)}} \right).
\end{align}

\hfill $\blacksquare$

\subsection*{Proof of Cor.~\ref{cor:mean-mmd-ipm}}

By definition, for any class of functions $\mathcal{F}$, the Integral Probability Metric (IPM) satisfies:
\begin{align}
    D_{\mathrm{IPM},\mathcal{F}}(P \| Q) = \sup_{f \in \mathcal{F}} \left| \mathbb{E}_{P}[f(Y)] - \mathbb{E}_{Q}[f(Y)] \right|.
\end{align}
If $f(Y) \in [a, b]$ for all $y \in \mathcal{Y}$, then for any probability measures $P, Q$:
\begin{align}
    \left| \mathbb{E}_{P}[f(Y)] - \mathbb{E}_{Q}[f(Y)] \right| \leq (b-a) D_{\mathrm{TV}}(P, Q).
\end{align}
For $\mathcal{F}_C \triangleq \{ f: \|f\|_{\infty} < C \}$, we have $f(y) \in (-C, C)$, so the range is $2C$. Consequently, 
\begin{align}
    D_{\mathrm{IPM},\mathcal{F}_C}(P_{a,x} \| Q_{a,x}) \leq 2C \cdot D_{\mathrm{TV}}(P_{a,x} \| Q_{a,x}).
\end{align}
From Corollary~\ref{cor:f-divergence}, we have $D_{\mathrm{TV}}(P_{a,x} \| Q_{a,x}) \leq 1 - e_a(x)$. Furthermore, by Pinsker's inequality and the KL bound from Corollary~\ref{cor:f-divergence}:
\begin{align}
    D_{\mathrm{TV}}(P_{a,x} \| Q_{a,x}) \leq \sqrt{\frac{1}{2} D_{\mathrm{KL}}(P_{a,x} \| Q_{a,x})} \leq \sqrt{-\frac{1}{2} \log e_a(x)}.
\end{align}
Combining these yields the result for IPM.

For MMD, let $\mathcal{H}_k$ be an RKHS with kernel $\mathtt{k}$ such that $\mathtt{k}(y, y) \le K$ for all $y$. For any $h \in \mathcal{H}_k$ with $\|h\|_{\mathcal{H}_k} \le 1$, we have $|h(y)| = |\langle h, \mathtt{k}_y \rangle| \le \|h\|_{\mathcal{H}_k} \sqrt{\mathtt{k}(y, y)} \le \sqrt{K}$. Thus, $h(y) \in [-\sqrt{K}, \sqrt{K}]$, and the range is $2\sqrt{K}$. Following similar logic:
\begin{align}
    D_{\mathrm{MMD},\mathtt{k}}(P_{a,x} \| Q_{a,x}) \leq 2\sqrt{K} \cdot D_{\mathrm{TV}}(P_{a,x} \| Q_{a,x}).
\end{align}
Using the TV bounds derived above, we obtain the MMD bound. \hfill $\blacksquare$

\subsection*{Proof of Prop.~\ref{prop:lower-upper}}
We will prove the following statement: For any arbitrary function $f$ over some space $\mathcal{X}$, the following holds: $\inf_{x \in \mathcal{X}} f(x) \;=\; - \sup_{x \in \mathcal{X}} \big(-f(x)\big)$. 
\begin{align}
    \inf_{x \in \mathcal{X}} f(x) \;=\; - \sup_{x \in \mathcal{X}} \big(-f(x)\big).
\end{align}
For any $x \in \mathcal{X}$, 
    \begin{align}
        -f(x) \leq -\inf_{x'} f(x'), \;\; \forall x \in \mathcal{X} \implies \inf_{x \in \mathcal{X}} f(x)  \leq -\sup_{x \in \mathcal{X}}(-f(x)). 
    \end{align}
    Also, by the definition of infimum, for any $\varepsilon > 0$, there exists $x_{\epsilon}$ such that 
    \begin{align}
        f(x_\varepsilon) \le \inf_{x \in \mathcal{X}} f(x) + \varepsilon.
    \end{align}
    Then, 
    \begin{align}
        -f(x_\varepsilon) \ge -\inf_x f(x) - \varepsilon \implies \sup_{x \in \mathcal{X}} (-f(x)) \ge -\inf_{x \in \mathcal{X}} f(x) - \varepsilon.
    \end{align}
    By taking $\epsilon \downarrow 0$, we have $\sup_{x \in \mathcal{X}} (-f(x)) \ge -\inf_{x\in\mathcal{X}} f(x)$. The proof is done by combining these two inequalities.  \hfill $\blacksquare$

\subsection*{Proof of Thm.~\ref{thm:representation-bound}}

Fix $(a,x)$ and write $P_{a,x}$ and $Q_{a,x}$ for the observational and interventional laws
on $(\mathcal{Y},\mathcal{F})$. By Assumption~\ref{assumption} (mutual absolute continuity),
the Radon--Nikodym derivative
\begin{align}
s(y) \;\coloneqq\; \frac{dQ_{a,x}}{dP_{a,x}}(y)
\end{align}
exists and satisfies $s(Y)>0$ $P_{a,x}$-a.s. For any measurable $\phi$ with
$\mathbb{E}_{Q_{a,x}}[|\phi(Y)|]<\infty$,
\begin{align}
\mathbb{E}_{Q_{a,x}}[\phi(Y)]
= \int \phi(y)\,Q_{a,x}(dy)
= \int \phi(y)\,s(y)\,P_{a,x}(dy)
= \mathbb{E}_{P_{a,x}}[s(Y)\phi(Y)].
\end{align}
Moreover, $\mathbb{E}_{P_{a,x}}[s(Y)]=\int dQ_{a,x}=1$. Define $g(s)\coloneqq s f(1/s)$
for $s>0$. Then
\begin{align}
\mathbb{E}_{P_{a,x}}[g(s(Y))]
= \int s(y) f(1/s(y))\,P_{a,x}(dy)
= \int f\!\left(\frac{dP_{a,x}}{dQ_{a,x}}(y)\right)\,Q_{a,x}(dy)
= D_f(P_{a,x}\|Q_{a,x}).
\end{align}
Hence the constraint $D_f(P_{a,x}\|Q_{a,x})\le \eta_f(a,x)$ is equivalent to
$\mathbb{E}_{P_{a,x}}[g(s(Y))]\le \eta_f(a,x)$, and the upper bound admits the primal form
\begin{align}
\theta_{\mathrm{up}}(a,x)
= \sup_{s>0}\Big\{ \mathbb{E}_{P_{a,x}}[s(Y)\phi(Y)] :
\mathbb{E}_{P_{a,x}}[s(Y)]=1,\;
\mathbb{E}_{P_{a,x}}[g(s(Y))]\le \eta_f(a,x)\Big\}.
\end{align}

This is a convex optimization problem (equivalently, minimize $-\mathbb{E}_{P_{a,x}}[s\phi]$)
with an affine equality and a convex inequality constraint. Slater's condition holds because
$s(\cdot)\equiv 1$ is feasible and satisfies
$\mathbb{E}_{P_{a,x}}[g(1)]=f(1)=0<\eta_f(a,x)$ (for $\eta_f(a,x)>0$). Therefore, strong
duality applies and the optimal value equals the dual optimal value.

Introduce Lagrange multipliers $u\in\mathbb{R}$ for $\mathbb{E}_{P_{a,x}}[s]=1$ and
$\lambda\ge 0$ for $\mathbb{E}_{P_{a,x}}[g(s)]\le \eta_f(a,x)$. The Lagrangian is
\begin{align}
\mathcal{L}(s,\lambda,u)
= \mathbb{E}_{P_{a,x}}[s(Y)\phi(Y)]
+u\bigl(1-\mathbb{E}_{P_{a,x}}[s(Y)]\bigr)
+\lambda\bigl(\eta_f(a,x)-\mathbb{E}_{P_{a,x}}[g(s(Y))]\bigr),
\end{align}
i.e.
\begin{align}
\mathcal{L}(s,\lambda,u)
= u + \lambda \eta_f(a,x)
+ \mathbb{E}_{P_{a,x}}\!\bigl[s(Y)(\phi(Y)-u) - \lambda g(s(Y))\bigr].
\end{align}
Thus
\begin{align}
\theta_{\mathrm{up}}(a,x)
= \inf_{\lambda\ge 0,\;u\in\mathbb{R}} \sup_{s>0} \mathcal{L}(s,\lambda,u).
\end{align}

For $\lambda>0$, define $t(Y)\coloneqq (\phi(Y)-u)/\lambda$. Using separability of the
integrand in $s(\cdot)$ and the standard interchange theorem for integral functionals
(equivalently, the conjugate-of-integral identity), we have
\begin{align}
\sup_{s>0}\mathbb{E}_{P_{a,x}}\!\bigl[s(Y)t(Y)-g(s(Y))\bigr]
= \mathbb{E}_{P_{a,x}}\!\Big[\sup_{s>0}\{s\,t(Y)-g(s)\}\Big]
= \mathbb{E}_{P_{a,x}}\!\bigl[g^*(t(Y))\bigr],
\end{align}
where $g^*(t)\coloneqq \sup_{s>0}\{st-g(s)\}$ is the convex conjugate of $g$.
Consequently,
\begin{align}
\sup_{s>0}\mathcal{L}(s,\lambda,u)
= u + \lambda \eta_f(a,x) + \lambda\,\mathbb{E}_{P_{a,x}}
\!\left[g^*\!\left(\frac{\phi(Y)-u}{\lambda}\right)\right].
\end{align}
Minimizing over $(\lambda,u)$ yields the stated dual representation:
\begin{align}
\theta_{\mathrm{up}}(a,x)
= \inf_{\lambda>0,\;u\in\mathbb{R}}
\left\{\lambda \eta_f(a,x) + u + \lambda\,\mathbb{E}_{P_{a,x}}
\!\left[g^*\!\left(\frac{\phi(Y)-u}{\lambda}\right)\right]\right\}.
\end{align}
\hfill $\blacksquare$

\subsection*{Proof of Prop.~\ref{prop:convex-conjugate-unified}}
    Substitute $r = 1/s$. Then, $st-g(s)=st-s f(1/s)=\frac{t-f(r)}{r}$. Taking $\sup_{s > 0}$ is the same as taking $\sup_{r > 0}$. Therefore, $g^{\ast}(t) = \sup_{r > 0} \frac{t-f(r)}{r}$.

    For the optimality condition, define
    \begin{align}
        H_t(r) \coloneqq \frac{t - f(r)}{r}, \qquad r>0.
    \end{align}
    Assume the supremum is attained at some $r^\ast>0$, and set
    \begin{align}
        v \coloneqq g^\ast(t) = H_t(r^\ast) = \frac{t - f(r^\ast)}{r^\ast}.
    \end{align}
    Then for every $r>0$,
    \begin{align}
        \frac{t - f(r)}{r} \le v
        \quad\Longleftrightarrow\quad
        f(r) \ge t - v r.
    \end{align}
    At $r=r^\ast$ we have equality: $f(r^\ast)=t-vr^\ast$. Hence for all $r>0$,
    \begin{align}
        f(r) \ge f(r^\ast) - v(r-r^\ast) = f(r^\ast) + a(r-r^\ast),
    \end{align}
    where $a\coloneqq -v$. By the supporting-hyperplane characterization of the convex subdifferential, this implies $a\in\partial f(r^\ast)$. Finally, $f(r^\ast)=t-vr^\ast$ gives
    \begin{align}
        t = f(r^\ast) - r^\ast a,
        \qquad
        g^\ast(t)=v=-a.
    \end{align}
    If $f$ is differentiable at $r^{\ast}$, then $\partial f(r^{\ast}) = \{f'(r^{\ast})\}$ and the conclusion follows. 

    \hfill $\blacksquare$

\subsection*{Proof of Coro.~\ref{cor:convex-conjugate}}
\paragraph{KL.} $f_{\mathrm{KL}}(r)= r\log r$. Then, 
\begin{align}
    g_{\mathrm{KL}}(s)=s f_{\mathrm{KL}}(1/s)=s\cdot \frac1s\log(1/s)= -\log s.
\end{align}
Now, compute $g^{\ast}_{\mathrm{KL}}(t)=\sup_{s >0}\{st+\log s\}$. Let $\psi(s)=st+\log s$. Then, $\psi'(s)=t+1/s$. If $t < 0$, the stationary point is $s^{\ast} = -1/t > 0$, giving $g^*_{\mathrm{KL}}(t)=\psi(s^*)=(-1/t)t+\log(-1/t)=-1-\log(-t)$. If $t \geq 0$, then $st+\log s\to\infty$ as $s \rightarrow \infty$, so $g^*_{\mathrm{KL}}(t)=+\infty$. 

\paragraph{Hellinger.} $f_H(r)=\tfrac12(\sqrt r-1)^2=\tfrac12(r-2\sqrt r+1)$. Then, 
\begin{align}
    g_H(s)=s f_H(1/s) =\frac{1}{2}\,s\Big(\frac{1}{s}-\frac{2}{\sqrt s}+1\Big)
=\frac{1}{2}(1-2\sqrt{s}+s).
\end{align}
Note $g_H^*(t)=\sup_{s >0}\Big\{st-\tfrac12(1-2\sqrt s+s)\Big\}$. Let $u \triangleq \sqrt{s} > 0$ so $s = u^2$. The objective becomes $F(u)=t u^2-\tfrac12(1-2u+u^2)
=\Big(t-\tfrac12\Big)u^2+u-\tfrac12$. If $t < 1/2$, $F$ is concave quadratic in $u$. Since $F'(u) = 2(t-1/2)u + 1$, $u^{\ast} = \frac{1}{1-2t}$. Plugging in, 
\begin{align}
    g_H^*(t)=F(u^*)
=\Big(t-\tfrac12\Big)\frac{1}{(1-2t)^2}+\frac{1}{1-2t}-\frac12
=\frac{t}{1-2t}.
\end{align}
If $t \geq 1/2$, then $F(u) \rightarrow \infty$ as $u \rightarrow \infty$, so $g^{\ast}_{\mathrm{H}}(t) = +\infty$.

\paragraph{$\chi^2$.} $g_{\chi^2}(s)=s f_{\chi^2}(1/s) =\frac12\,s\Big(\frac1s-1\Big)^2 =\frac{(1-s)^2}{2s}$. Also, $g^*_{\chi^2}(t)=\sup_{s > 0}\Big\{st-\frac{(1-s)^2}{2s}\Big\}$, where $\frac{(1-s)^2}{2s}=\frac12\Big(\frac1s-2+s\Big)$. Then, the objective is 
\begin{align}
    st-\frac12\Big(\frac1s-2+s\Big) =1+s\Big(t-\tfrac12\Big)-\frac{1}{2s}.
\end{align}
Differentiate w.r.t. $s$: 
\begin{align}
    \frac{d}{ds}\Big(1+s(t-\tfrac12)-\frac{1}{2s}\Big) =(t-\tfrac12)+\frac{1}{2s^2}.
\end{align}
Plugging in (using $1/s^*=\sqrt{1-2t}$): 
\begin{align}
    g^*_{\chi^2}(t) =1+s^*(t-\tfrac12)-\frac{1}{2s^*} =1-\frac{\sqrt{1-2t}}{2}-\frac{\sqrt{1-2t}}{2} =1-\sqrt{1-2t}.
\end{align}
At $t = 1/2$, this becomes $1$. If $t > 1/2$, the term $s(t-\tfrac12)$ drives the supremum to $+\infty$ as $s \rightarrow \infty$. 

\paragraph{TV.} $g_{\mathrm{TV}}(s)=s f_{\mathrm{TV}}(1/s) =\frac12\,s\Big|\frac1s-1\Big| =\frac12|1-s|$. Also, $g_{\mathrm{TV}}^*(t)=\sup_{s > 0}\Big\{st-\tfrac12|1-s|\Big\}$. Split this into two regions, where $s \geq 1$ and $0 < s \leq 1$. 

When $s \geq 1$, $|1-s| = s-1$. So, 
\begin{align}
    st-\tfrac12(s-1)=s\Big(t-\tfrac12\Big)+\tfrac12.
\end{align}
If $t > 1/2$: this goes to $+\infty$ as $s \rightarrow \infty$. If $t \leq 1/2$, the maximum over $s \geq 1$ occurs at the smallest $s$; i.e., $s=1$, giving value $t$. 

When $0 < s \leq 1$, $|1-s| = 1-s$, so 
\begin{align}
    st-\tfrac12(1-s)=s\Big(t+\tfrac12\Big)-\tfrac12.
\end{align}
If $t < -1/2$, then its maximum is $-1/2$. If $t \geq -1/2$, then it's maximized at $s=1$, giving value $t$. As a result, 
\begin{align}
    g_{\mathrm{TV}}^*(t)=
\begin{cases}
-\tfrac12, &\text{ if } t\le -\tfrac12,\\
t, &\text{ if } -\tfrac12 < t \leq \tfrac12,\\
+\infty, &\text{ if } t> \tfrac12.
\end{cases}
\end{align}

\hfill $\blacksquare$

\paragraph{Jensen-Shannon.}
$g_{\mathrm{JS}}(s)=\frac12\Big(s\log s-(1+s)\log(1+s)+(1+s)\log 2\Big)$. To compute $g_{\mathrm{JS}}^\ast(t)=\sup_{s > 0}\{st-g_{\mathrm{JS}}(s)\}$, let $F(s) = st - g_{\mathrm{JS}}(s)$. 

We have 
\begin{align}
    g'_{\mathrm{JS}}(s)=\frac12\Big(\log s-\log(1+s)+\log2\Big) =\frac12\log\Big(\frac{2s}{1+s}\Big).
\end{align}
Set $F'(s)=0$, which means $t = g'_{\mathrm{JS}}(s)$; i.e., 
\begin{align}
    2t=\log\Big(\frac{2s}{1+s}\Big)
\quad\Longleftrightarrow\quad
e^{2t}=\frac{2s}{1+s}.
\end{align}
Solving this for $s$ gives
\begin{align}
    e^{2t}(1+s)=2s
\;\Rightarrow\;
s^*=\frac{e^{2t}}{2-e^{2t}}.
\end{align}
This requires $2-e^{2t} > 0$, i.e., $t < \frac{1}{2}\log 2$. If $t \geq \frac{1}{2}\log 2$, the objective grows like $s(t-\frac{1}{2}\log 2)$ for large $s$, hence the supremum is $+\infty$. 

Now evaluate the objective at $s^{\ast}$. Let $z \triangleq e^{2t}$ so that $s^{\ast} = z/(2-z)$ and $1+s^{\ast} = 2/(2-z)$. Then, 
\begin{align}
    \log s^*=\log z-\log(2-z),\qquad \log(1+s^*)=\log2-\log(2-z).
\end{align}
Plug into $g_{\mathrm{JS}}(s)$: 
\begin{align}
    g_{\mathrm{JS}}(s^*)
=\frac12\Big(s^*\log s^*-(1+s^*)\log(1+s^*)+(1+s^*)\log2\Big)
=\frac12\Big(s^*\log z+\log(2-z)\Big).
\end{align}
Since $\log z = 2t$, this is $g_{\mathrm{JS}}(s^*)=t s^*+\frac12\log(2-e^{2t})$. Therefore, 
\begin{align}
    g_{\mathrm{JS}}^*(t)
= s^* t - g_{\mathrm{JS}}(s^*)
= -\frac12\log(2-e^{2t}),
\qquad \text{for } t<\tfrac12\log2,
\end{align}
and $g_{\mathrm{JS}}^*(t)=+\infty$ otherwise. \hfill $\blacksquare$

\subsection*{Proof of Prop.~\ref{prop:justification}}

($(1) \implies (2)$). For each fixed $(a,x)$, define
\begin{align}
    \Delta(a,x) \triangleq \ell(h^{\star}, u^{\star}; a,x) - \operatorname*{ess\,inf}_{h,u \in \mathcal{F}}\ell(h,u; a,x).
\end{align}
Assume, for contradiction, that $(2)$ fails; i.e., $P_{A,X}(B) > 0$ for $B \triangleq \{(a,x): \Delta(a,x) > 0\}$. By the definition of the essential infimum and the decomposability of $\mathcal{F}$, there exists a measurable pair $(\tilde{h},\tilde{u}) \in \mathcal{F}$ such that $\ell(\tilde{h},\tilde{u}; a,x) < \ell(h^{\star}, u^{\star}; a,x)$ on a set of positive measure $B' \subseteq B$. 

Define $h'(a,x) \triangleq \tilde{h}(a,x)\boldsymbol{1}((a,x) \in B') + h^{\star}(a,x)\boldsymbol{1}((a,x) \not\in B') $ and define $u'(a,x)$ similarly. By the decomposability assumption, $(h',u') \in \mathcal{F}$. Then, 
\begin{align}
    \mathcal{R}(h^{\star},u^{\star}) &= \mathbb{E}[\ell(h^{\star},u^{\star},A,X)\boldsymbol{1}((A,X) \not \in B')] + \mathbb{E}[\ell(h^{\star},u^{\star},A,X)\boldsymbol{1}((A,X) \in B')] \\ 
    &> \mathbb{E}[\ell(h^{\star},u^{\star},A,X)\boldsymbol{1}((A,X) \not \in B')] + \mathbb{E}[\ell(\tilde{h},\tilde{u},A,X)\boldsymbol{1}((A,X) \in B')] \\ 
    &= \mathcal{R}(h', u').
\end{align}
This contradicts the optimality of $(h^{\star}, u^{\star})$ in (1). Therefore, $P_{A,X}(B) = 0$; i.e., $(h^{\star}, u^{\star})$ is a minimizer of $\ell(h, u; a,x)$ for $P_{A,X}$-almost every $(a,x)$.

($(2) \implies (1)$). Since $\ell(h^{\star}, u^{\star}; a,x) \le \ell(h,u;a,x)$ for all $(h,u) \in \mathcal{F}$ and for $P_{A,X}$-almost every $(a,x)$, integrating yields $\mathcal{R}(h^{\star}, u^{\star}) \le \mathcal{R}(h,u)$ for all $(h,u) \in \mathcal{F}$. \hfill $\blacksquare$

\subsection*{Proof of Lemma~\ref{lemma:orthogonality}}
Define 
\begin{align}
    e_a &\triangleq \Pr(A=a \mid X) \\ 
    \lambda_{a} &\triangleq \exp(h_{\beta}(a,X)) \\ 
    B_a(e) &\triangleq B_f(e_a(X)) \\ 
    u_a &\triangleq u(a,X) \\ 
    g^{\ast}_a &\triangleq g^{\ast}\left( \frac{\varphi(Y) - u(A,X)}{\lambda_A(X)} \right).
\end{align}
Then, 
\begin{align}
    \ell(V;(\beta,\gamma),e) \triangleq \lambda_A(B_A + g^{\ast}_A) + u_A(X). 
\end{align}
Define 
\begin{align}
    L_1(e) &\triangleq  \lambda_A(B_A + g^{\ast}_A) + u_A(X). 
\end{align}
The correction term is 
\begin{align}
    L_2(e) &\triangleq \sum_{a} e_a \lambda_a B'_a \{\boldsymbol{1}(A=a) - e_a\}. 
\end{align}
Then, $\mathcal{R}^{\mathrm{db}}(e) \triangleq \mathbb{E}[L_1(e) + L_2(e)]$. Then,
\begin{align}
    \frac{\partial}{\partial t}\mathbb{E}[L_1(e_t)]\bigg\vert_{t=0} &= \frac{\partial}{\partial t}\mathbb{E}[L_1(e_A + ts_A)]\bigg\vert_{t=0} \\ 
    &= \mathbb{E}\left[\lambda_A B'(e_A) s_A \right]. 
\end{align}
Also, 
\begin{align}
    L_2(e) &\triangleq \sum_{a} \underbrace{e_a \lambda_a B'_a}_{U_a(e_a)} \underbrace{\{\boldsymbol{1}(A=a) - e_a\}}_{V_a(e_a)}. 
\end{align}
Then, 
\begin{align}
    \frac{\partial}{\partial t}\mathbb{E}[L_2(e_t)]\bigg\vert_{t=0} &= \mathbb{E}\left[\sum_{a \in \mathcal{A}} \left( \frac{\partial U_a}{\partial t}V_a + \frac{\partial V_a}{\partial t}U_a \right)  \right], 
\end{align}
where 
\begin{align}
    \mathbb{E}\left[ \sum_{a \in \mathcal{A}}U'_a(e)V_a(e_a) \right] &= \mathbb{E}_{X}\left[ \sum_{a \in \mathcal{A}}U'_a(e) \mathbb{E}_{A \mid X}[\boldsymbol{1}(A=a) - e_a] \right]  = 0,
\end{align}
and 
\begin{align}
    \mathbb{E}\left[ \sum_{a \in \mathcal{A}}U_a(e)V'_a(e_a) \right] &= - \mathbb{E}\left[ \sum_{a \in \mathcal{A}}U_a(e)s_a \right] = - \mathbb{E}_{X}\left[ \sum_{a \in \mathcal{A}}e_a\lambda_a B'_a s_a \right]. 
\end{align}
Then, 
\begin{align}
    \frac{\partial R^{\mathrm{db}}}{\partial t} &= \mathbb{E}[\lambda_A B'(e_A)s_A] - \mathbb{E}_{X}\left[ \sum_{a \in \mathcal{A}}e_a\lambda_a B'_a s_a \right] \\ 
    &= \mathbb{E}_X\left[\sum_{a \in \mathcal{A}}e_a\lambda_a B'_a s_a\right] - \mathbb{E}_{X}\left[ \sum_{a \in \mathcal{A}}e_a\lambda_a B'_a s_a \right] \\ 
    &= 0. 
\end{align}
\hfill $\blacksquare$

\subsection*{Proof of Theorem~\ref{thm:error-analysis}}

\begin{lemma}[\textbf{Higher-order smoothness} $\Rightarrow$ \textbf{Local quadratic expansion inequality}]\label{lemma:equivalent}
    \normalfont 
    Higher-order smoothness in Assumption~\ref{assumption:regularity-condition} implies the local quadratic expansion inequality: 
    \begin{align}\label{eq:local-quadratic}
        \frac{\kappa_1}{2}\|\vartheta-\vartheta_0\|^2  \;\le\; \mathcal R^{\mathrm{db}}(\vartheta;e_0)-\mathcal R^{\mathrm{db}}(\vartheta_0;e_0) \;\le\; \frac{\kappa_2}{2}\|\vartheta-\vartheta_0\|^2, \text{ for } \vartheta \in \varTheta_0.
    \end{align}
\end{lemma}
\begin{proof}[\textbf{Proof of Lemma~\ref{lemma:equivalent}}]
    Let $r(t) \triangleq \mathcal{R}(\vartheta_t; e_0)$, where $\vartheta_t \triangleq \vartheta_0 + t(\vartheta - \vartheta_0)$ for $t \in [0,1]$. By Taylor's theorem with integral remainder,
    \begin{align}
        r(1) = r(0) + r'(0) + \int_{0}^{1}(1-t) r''(t) dt.
    \end{align}
    Since $\vartheta_0$ is a local minimizer, $r'(0) = (\vartheta- \vartheta_0)^{\intercal}\nabla_{\vartheta}\mathcal{R}(\vartheta_0; e_0) = 0$. The second derivative is
    \begin{align}
        r''(t) = (\vartheta- \vartheta_0)^{\intercal} H(\vartheta_t; e_0) (\vartheta- \vartheta_0).
    \end{align}
    Under the Higher-order smoothness assumption ($\kappa_1 I \preceq H(\vartheta; e_0) \preceq \kappa_2 I$ for $\vartheta \in \varTheta_0$), and assuming convexity of $\varTheta_0$ so that the path lies in $\varTheta_0$, we have
    \begin{align}
       \frac{\kappa_1}{2} \|\vartheta - \vartheta_0 \|^2_2 \leq \int_{0}^{1}(1-t) (\vartheta - \vartheta_0)^{\intercal} H(\vartheta_t; e_0) (\vartheta- \vartheta_0) dt \leq \frac{\kappa_2}{2} \|\vartheta - \vartheta_0 \|^2_2.
    \end{align}
\end{proof}

\subsubsection*{Proof of Eq.~\eqref{eq:thm:error-analysis-1}}

For brevity, we write $R(\vartheta; e') \triangleq R^{\mathrm{db}}(\vartheta; e')$ for any $\vartheta$ and $e'$. Let $\widehat{R}_{k}$ denote the empirical risk of $R$ using the $k$'th fold dataset. 

We decompose the population excess risk using a telescoping sum:
\begin{align}
    R(\widehat{\vartheta}_k; e_0) - R(\vartheta_0; e_0)  &= \underbrace{ R(\widehat{\vartheta}_k; e_0) - R(\widehat{\vartheta}_k; \widehat{e}^{k}) }_{(A)} 
            \;+\; \underbrace{ R(\widehat{\vartheta}_k; \widehat{e}^{k}) - \widehat{R}_k (\widehat{\vartheta}_k; \widehat{e}^{k}) }_{(B)} \\
            &\quad + \underbrace{ \widehat{R}_k (\widehat{\vartheta}_k; \widehat{e}^{k}) - \widehat{R}_k (\vartheta_0; \widehat{e}^{k}) }_{\leq 0} 
            \;+\; \underbrace{ \widehat{R}_k (\vartheta_0; \widehat{e}^{k}) - R(\vartheta_0; \widehat{e}^k) }_{(C)}\\ 
            &\quad + \underbrace{ R(\vartheta_0; \widehat{e}^k) - R(\vartheta_0; e_0) }_{(D)}.
\end{align}
The term $\leq 0$ is due to the optimality of $\widehat{\vartheta}_k$ for the empirical risk objective. We will show that 
\begin{enumerate}
    \item $(B) + (C) = O_p(n^{-1/2})$ by the uniform LLN in Assumption~\ref{assumption:regularity-condition}. 
    \item $(A) + (D) = O_p(r^2_n)$ by the orthogonality and smoothness in Assumption~\ref{assumption:regularity-condition}. 
\end{enumerate}
As a result, 
\begin{align}
    R(\widehat{\vartheta}_k; e_0) - R(\vartheta_0; e_0)  &= O_p(n^{-1/2}) + O_p(r^2_n).
\end{align}

\paragraph{Bounds for $(B) + (C)$.}
Terms $(B) + (C)$ are bounded by Uniform LLN as follows: 
\begin{align}
    (B) + (C)  \le 2\sup_{\vartheta}\big|R(\vartheta;\widehat e^k) - \widehat R_k(\vartheta;\widehat e^k)\big| = O_p(n^{-1/2}).
\end{align}

\paragraph{Bounds for $(A) + (D)$.}
Assume that the risk functional $e \mapsto R(\vartheta; e)$ is twice Fréchet differentiable with bounded second derivatives on the positivity region. Fix a $\vartheta$. Consider a parametric submodel $t \mapsto e^t \triangleq e_0 + t(\widehat{e}^k - e_0)$. Let $\delta e_0 \triangleq \widehat{e}^k - e_0$.

By Taylor's theorem, there exists $e^{\dagger}$ between $e_0$ and $\hat{e}^{k}$ such that:
\begin{align}
    R(\vartheta; \widehat{e}^k) = R(\vartheta; e_0) + \nabla_{e}R(\vartheta; e_0)[\delta e_0] + \frac{1}{2}\nabla_{ee}R(\vartheta; e^{\dagger})[\delta e_0, \delta e_0].
\end{align}
Rearranging for term (D) where $\vartheta = \vartheta_0$:
\begin{align}
    (D) = R(\vartheta_0; \widehat{e}^k) - R(\vartheta_0; e_0) = \nabla_{e}R(\vartheta_0; e_0)[\delta e_0] + \frac{1}{2}\nabla_{ee}R(\vartheta_0; e^{\dagger})[\delta e_0, \delta e_0].
\end{align}
By Lemma~\ref{lemma:orthogonality} (Orthogonality), $\nabla_{e}R(\vartheta_0; e_0)[\delta e_0] = 0$. Using the boundedness of $\nabla_{ee}R$ (Assumption~\ref{assumption:regularity-condition}), we have $(D) = O_P(\| \widehat{e}^k - e_0 \|^2_2) = O_P(r^2_n)$.

For term (A) where $\vartheta = \widehat{\vartheta}_k$:
\begin{align}
    (A) = R(\widehat{\vartheta}_k; e_0) - R(\widehat{\vartheta}_k; \widehat{e}^{k}) = -\nabla_{e}R(\widehat{\vartheta}_k; e_0)[\delta e_0] + O_P(r_n^2).
\end{align}
Crucially, Lemma~\ref{lemma:orthogonality} states that orthogonality holds for \emph{all} $\vartheta$ (not just $\vartheta_0$). Therefore, $\nabla_{e}R(\widehat{\vartheta}_k; e_0)[\delta e_0] = 0$ directly. This implies that the first-order error term vanishes exactly, and we are left only with the second-order remainder:
\begin{align}
    (A)  = O_P(r^2_n).
\end{align}
Combining yields:
\begin{align}
    (A) + (D) = O_P(r^2_n). 
\end{align}

\paragraph{Bound Derivation.}
Combining all terms:
\begin{align}
    R(\widehat{\vartheta}_k; e_0) - R(\vartheta_0; e_0) = O_p(n^{-1/2}) + O_P(r^2_n). 
\end{align}
Assuming consistency (so $\widehat{\vartheta}_k \in \varTheta_0$ w.h.p), we apply Lemma~\ref{lemma:equivalent}:
\begin{align}
    \frac{\kappa_1}{2}\|\widehat{\vartheta}_k - \vartheta_0 \|^2 \leq  R(\widehat{\vartheta}_k; e_0) - R(\vartheta_0; e_0).
\end{align}
Solving the quadratic inequality for $\|\widehat{\vartheta}_k - \vartheta_0 \|$ establishes:
\begin{align}
    \|\widehat{\vartheta}_k - \vartheta_0 \|^2 = O_p(n^{-1/2} + r^2_n).
\end{align}

\subsubsection*{Proof of Eq.~\eqref{eq:thm:error-analysis-2}}

For brevity, we just write 
\begin{align}
    \theta_k \triangleq \widehat{\theta}^{(k)}_{\varphi}, \quad \lambda_k \triangleq \widehat{\lambda}_k, \quad \eta_k \triangleq \widehat{\eta}^k_f, \quad m_k \triangleq \widehat{m}_k, \quad u_k \triangleq \widehat{u}_k. 
\end{align}
All the true parameters are indexed as $0$. For each $(a,x)$, 
\begin{align}
    \theta_k(a,x) - \overline{\theta}_{\varphi}(a,x) &= \underbrace{ (\lambda_k - \lambda_0) (\eta_0 + m_0) }_{(I)} 
                \;+\; \underbrace{ (\lambda_k - \lambda_0) (\eta_k - \eta_0) }_{(II)} \\ 
                &\quad + \underbrace{ \lambda_0(\eta_k - \eta_0) }_{(III)} 
                \;+\; \underbrace{ \lambda_k(m_k - m_0)  }_{(IV)}
                \;+\; \underbrace{ (u_k - u_0) }_{(V)}.
\end{align}

We bound the squared $L_2$ norm of each term. By Lipschitz parametrization (Assumption~\ref{assumption:regularity-condition-2}):
\begin{align}
    \| (V) \|^2_2 = O_P(\| \widehat{\vartheta}_k - \vartheta_0 \|^2_2).
\end{align}

For $(I)$, using the boundedness of nuiances (Assumption~\ref{assumption:regularity-condition-2}, $|\eta_0+m_0| \le C$):
\begin{align}
    \| (I) \|^2_2 \leq C^2 \| \lambda_k - \lambda_0\|^2_2 \leq C' \| \widehat{\vartheta}_k - \vartheta_0 \|^2_2 = O_P(\| \widehat{\vartheta}_k - \vartheta_0 \|^2_2).
\end{align}

For $(III)$, using $|\lambda_0| \le e^{M}$ and Lipschitz continuity of $\eta$ (via $B_f$) with respect to $e$:
\begin{align}
    \|\text{(III)}\|^2_2 \le e^{2M} \|\eta_k - \eta_0\|_2^2 = O_P(r^2_n).
\end{align}

For $(II)$, we use the supremum bound on $\lambda$: $\|\lambda_k - \lambda_0\|_\infty \le 2e^M$. Then:
\begin{align}
    \|\text{(II)}\|_2^2 \le \|\lambda_k-\lambda_0\|_\infty^2 \|\eta_k-\eta_0\|_2^2 \le 4e^{2M} r_n^2 = O_P(r_n^2). 
\end{align}

Finally, consider $(IV)$. Define $m_{\widehat{\varphi}}(a,x) \triangleq \mathbb{E}[Z^k_i \mid A=a,X=x]$. 
Decompose $m_k - m_0 = (m_k - m_{\widehat{\vartheta}}) + (m_{\widehat{\vartheta}} - m_0)$.
By Assumption~\ref{assumption:regularity-condition-2}, $\|m_k - m_{\widehat{\vartheta}}\|_2  = O_P(s_n)$. By Lipschitz, $\|m_{\widehat{\vartheta}} - m_0\|_2 \leq L_m \|\widehat{\vartheta} - \vartheta_0 \|$.
Therefore, 
\begin{align}
    \|(IV) \|^2_2 \leq e^{2M}(\|m_k - m_{\widehat{\vartheta}}\|_2 + \|m_{\widehat{\vartheta}} - m_0\|_2)^2 = O_P(s^2_n) + O_P(\| \widehat{\vartheta}_k - \vartheta_0 \|^2).
\end{align}

Combining all terms shows that 
\begin{align}
    \| \theta_k - \overline{\theta}_{\varphi} \|^2_2 = O_P(n^{-1/2} + r^2_n + s^2_n).
\end{align}

\subsection*{Proof of Lemma~\ref{lemma:valid-coverage}}

Let the sorted elements of $\widehat{\boldsymbol{\theta}}_{\mathrm{up}}$ be denoted by $u_{(1)} \le u_{(2)} \le \dots \le u_{(n_f)}$. By Definition~\ref{def:kth}, $\widehat{\theta}^{k}_{\mathrm{up}} = u_{(k)}$. The inequality $u_{(k)} \ge \theta$ holds if and only if at least $n_f - k + 1$ elements satisfy $u_i \ge \theta$ (since this is equivalent to having at most $k-1$ elements strictly less than $\theta$).

Similarly, let the sorted elements of $\widehat{\boldsymbol{\theta}}_{\mathrm{lo}}$ be $l_{(1)} \le \dots \le l_{(n_f)}$. By definition, $\widehat{\theta}^{k}_{\mathrm{lo}}$ is the $k$-th largest element, which corresponds to $l_{(n_f - k + 1)}$. The inequality $l_{(n_f - k + 1)} \le \theta$ holds if and only if at least $n_f - k + 1$ elements satisfy $l_i \le \theta$ (since this is equivalent to having at most $k-1$ elements strictly greater than $\theta$). \hfill $\blacksquare$

\subsection*{Proof of Thm.~\ref{thm:error-ate}}
Write $\widehat R\equiv \widehat R_n$. Decompose the population excess risk:
\begin{align*}
0 \le R(\widehat\vartheta; e_0) - R(\vartheta_0; e_0)
&= \underbrace{R(\widehat\vartheta;e_0)-R(\widehat\vartheta;\widehat e)}_{(A)}
 +\underbrace{R(\widehat\vartheta;\widehat e)-\widehat R(\widehat\vartheta;\widehat e)}_{(B)} \\
&\quad+\underbrace{\widehat R(\widehat\vartheta;\widehat e)-\widehat R(\vartheta_0;\widehat e)}_{\le 0}
 +\underbrace{\widehat R(\vartheta_0;\widehat e)-R(\vartheta_0;\widehat e)}_{(C)}
 +\underbrace{R(\vartheta_0;\widehat e)-R(\vartheta_0;e_0)}_{(D)} .
\end{align*}

By the uniform LLN in Assumption~\ref{assumption:regularity-condition-3},
\[
(B)+(C) \le 2\sup_{\vartheta\in\Theta}\big|\widehat R(\vartheta;\widehat e)-R(\vartheta;\widehat e)\big|
= O_p(n^{-1/2}).
\]
By Lipschitz continuity of $R(\vartheta;\cdot)$ in $e$ uniformly over $\vartheta\in\Theta$,
\[
|(A)|+|(D)|
\le 2L_R\|\widehat e-e_0\|_1
= O_p(n^{-1/2}),
\]
since $\widehat e_a=n_a/n$ implies $\|\widehat e-e_0\|_1=O_p(n^{-1/2})$ under positivity.

Hence
\[
0 \le R(\widehat\vartheta;e_0)-R(\vartheta_0;e_0)=O_p(n^{-1/2}).
\]
Let $\Theta_0$ be the neighborhood from the quadratic growth condition (Lemma~\ref{lemma:equivalent}).
Since $R(\vartheta;e_0)-R(\vartheta_0;e_0)$ is bounded away from $0$ on $\Theta\setminus\Theta_0$,
the above display implies $\Pr(\widehat\vartheta\in\Theta_0)\to 1$. Therefore, on this event,
\[
\frac{\kappa_1}{2}\|\widehat\vartheta-\vartheta_0\|_2^2
\le R(\widehat\vartheta;e_0)-R(\vartheta_0;e_0)
= O_p(n^{-1/2}),
\]
so $\|\widehat\vartheta-\vartheta_0\|_2^2=O_p(n^{-1/2})$.

Next, write (as in Def.~\ref{def:bound-average}, marginal case)
\[
\widehat\theta_\varphi(a)=\widehat\lambda_a\big(\widehat\eta_a+\widehat m_a\big)+\widehat u_a,
\qquad
\overline\theta_\varphi(a)=\lambda_{0,a}\big(\eta_{0,a}+m_{0,a}\big)+u_{0,a},
\]
where $\lambda_a=\exp(h_a)$, $\eta_a=B_f(e_a)$,
$Z_\vartheta \equiv g^\ast\!\big((\varphi(Y)-u_A)/\lambda_A\big)$,
$m_{\vartheta,a}=\mathbb E[Z_\vartheta\mid A=a]$, and $\widehat m_a=n_a^{-1}\sum_{i:A_i=a} Z_{\widehat\vartheta,i}$.
Decompose, for each $a$,
\begin{align*}
\widehat\theta_\varphi(a)-\overline\theta_\varphi(a)
&= (\widehat\lambda_a-\lambda_{0,a})(\eta_{0,a}+m_{0,a})
+(\widehat\lambda_a-\lambda_{0,a})(\widehat\eta_a-\eta_{0,a})
+\lambda_{0,a}(\widehat\eta_a-\eta_{0,a}) \\
&\quad+\widehat\lambda_a(\widehat m_a-m_{0,a})
+(\widehat u_a-u_{0,a})
\;\;=:(I)+(II)+(III)+(IV)+(V).
\end{align*}

By boundedness of $h$ and smoothness of $\exp(\cdot)$ on bounded sets,
$\|\widehat\lambda-\lambda_0\|_2 \lesssim \|\widehat h-h_0\|_2\le\|\widehat\vartheta-\vartheta_0\|_2$,
and $\|\widehat u-u_0\|_2\le\|\widehat\vartheta-\vartheta_0\|_2$.
Thus $\|(I)\|_2^2+\|(V)\|_2^2 = O_p(\|\widehat\vartheta-\vartheta_0\|_2^2)=O_p(n^{-1/2})$.

Also, $\|\widehat\eta-\eta_0\|_2 \lesssim \|\widehat e-e_0\|_1 = O_p(n^{-1/2})$ (bounded $B_f'$),
so $\|(III)\|_2^2=O_p(n^{-1})$.
Moreover, $\|(II)\|_2 \le \|\widehat\lambda-\lambda_0\|_2\|\widehat\eta-\eta_0\|_\infty
= O_p(n^{-1/4})\cdot O_p(n^{-1/2})=O_p(n^{-3/4})$, hence $\|(II)\|_2^2=O_p(n^{-3/2})$.

For $(IV)$, decompose
\[
\widehat m_a-m_{0,a}
= \underbrace{\frac{1}{n_a}\sum_{i:A_i=a}\big(Z_{\widehat\vartheta,i}-Z_{\vartheta_0,i}\big)}_{(a)}
+\underbrace{\left\{\frac{1}{n_a}\sum_{i:A_i=a}Z_{\vartheta_0,i}-\mathbb E[Z_{\vartheta_0}\mid A=a]\right\}}_{(b)}
+\underbrace{\big(m_{\vartheta_0,a}-m_{\widehat\vartheta,a}\big)}_{(c)}.
\]
By bounded derivative of $g^\ast$ and bounded parameters, $Z_\vartheta$ is Lipschitz in $\vartheta$,
so $(a)=O_p(\|\widehat\vartheta-\vartheta_0\|_2)=O_p(n^{-1/4})$ and $(c)=O_p(\|\widehat\vartheta-\vartheta_0\|_2)=O_p(n^{-1/4})$.
By positivity $n_a\asymp n$ and CLT, $(b)=O_p(n^{-1/2})$.
Hence $\|\widehat m-m_0\|_2=O_p(n^{-1/4})$. Since $\widehat\lambda$ is bounded,
$\|(IV)\|_2^2=O_p(n^{-1/2})$.

Collecting terms, the dominant squared contributions are $O_p(n^{-1/2})$ from $(I)$, $(IV)$, and $(V)$,
so $\|\widehat\theta_\varphi-\overline\theta_\varphi\|_2^2=O_p(n^{-1/2})$.
\hfill$\blacksquare$

\endgroup

\end{document}